\documentclass[10pt,twocolumn,letterpaper]{article}

\usepackage{times}
\usepackage{epsfig}
\usepackage{graphicx}
\usepackage{amsmath}
\usepackage{amssymb}
\usepackage{algorithm}
\usepackage{algorithmic}
\usepackage[utf8]{inputenc}
\usepackage[style=base]{subcaption}
\usepackage{booktabs}
\usepackage{array}
\usepackage{float}
\usepackage{afterpage}
\usepackage{setspace}

\usepackage[pagebackref=true,breaklinks=true,colorlinks,bookmarks=false]{hyperref}

\newtheorem{theorem}{Theorem}
\newtheorem{definition}{Definition}
\newtheorem{proof}{Proof}

\newcommand{\etal}{\textit{et al.}}
\newcommand{\ie}{\textit{i.e.}}
\newcommand{\eg}{\textit{e.g.}}

\DeclareMathOperator*{\argmax}{arg\,max}
\DeclareMathOperator*{\argmin}{arg\,min}

\begin{document}

%%%%%%%%% TITLE
%%%%%%%%% TITLE
\title{GAN-based Projector for Faster Recovery with Convergence Guarantees in Linear Inverse Problems}

\author{Ankit Raj \thanks{Equal contribution. Ankit Raj and Yoram Bresler's research work was supported in part by the National Science Foundation under Grant IIS 14-47879 . Yuqi Li and Yoram Bresler's reseach work was supported in part by Sandia National Laboratories under Grant ID: AE056, IP: 00371547} \quad Yuqi Li\footnotemark[1] \quad Yoram Bresler
\\
University of Illinois at Urbana-Champaign, USA\\
{\tt\small \{ankitr3, yuqil3, ybresler\}@illinois.edu}
}

\maketitle
%%%%%%%%% ABSTRACT
\begin{abstract}
 \textit{ A Generative Adversarial Network (GAN) with generator $G$ trained to model the prior of images has been shown to perform better than sparsity-based regularizers in ill-posed inverse problems. Here, we propose a new method of deploying a GAN-based prior to solve linear inverse problems using projected gradient descent (PGD). Our method learns a network-based projector for use in the PGD algorithm, eliminating expensive computation of the Jacobian of $G$. Experiments show that our approach provides  a speed-up of $60\text{-}80\times$ over  earlier GAN-based recovery methods along with better accuracy. Our main theoretical result is that if the measurement matrix is moderately conditioned on the manifold range($G$) and the projector is $\delta$-approximate, then the algorithm is  guaranteed to reach  $O(\delta)$ reconstruction error in $O(log(1/\delta))$ steps in the low noise regime. Additionally, we propose a fast method to design such measurement matrices for a given $G$. Extensive experiments demonstrate the efficacy of this method by requiring $5\text{-}10\times$ fewer measurements than random Gaussian measurement matrices for comparable recovery performance.
 % Added one line
Because the learning of the GAN and projector is decoupled from the measurement operator, our GAN-based projector and recovery algorithm are applicable without retraining to all linear inverse problems, as confirmed by experiments on compressed sensing, super-resolution, and inpainting.
}
\end{abstract}

%%%%%%%%% BODY TEXT
\vspace*{-0.5em}
\section{Introduction}
\vspace*{-0.5em}

Many application such as computational imaging, and remote sensing fall in the compressive sensing (CS) paradigm. CS \cite{donoho2006compressed,candes2006stable} refers to projecting a high dimensional, sparse or sparsifiable signal $x\in\mathbb{R}^n$ to a lower dimensional measurement $y\in\mathbb{R}^m, m\ll n$, using a small set of linear, non-adaptive frames. The noisy measurement model is:
\begin{equation}
    y = Ax+v,A\in\mathbb{R}^{m\times n}, v\sim \mathcal{N}(0,\sigma^2 I)
\end{equation}
where the measurement matrix $A$ is often a random matrix. In this work, we are interested in the problem of recovering the unknown natural signal $x$, from the compressed measurement $y$, given the measurement matrix $A$. Traditionally, for signal priors, natural images are considered sparse in some fixed or learnable basis \cite{elad2006image, dong2011image, wen2015structured,liu2017image,dabov2009bm3d,yang2010image, elad2010sparse, li2009user}. \\
Instead of the sparse prior commonly adopted by CS literature, we turn to a learned prior. 
Neural network-based inverse problem solvers have been explored recently \cite{gregor2010learning,venkatakrishnan2013plug,rick2017one,adler2017solving,fan2017inner,gupta2018cnn,mardani2018neural,romano2017little,liu2017image, wen2019transform, metzler2017learned}. However, \cite{adler2017solving,fan2017inner,gupta2018cnn,mardani2018neural} use information about the measurement matrix $A$ while training the network. Thus, their algorithms are limited to a particular set-up to solve specific inverse-problem and usually cannot solve other problems without retraining. Another line of work, \cite{mousavi2017deepcodec, mousavi2018data} jointly optimizes the measurement matrix and recovery algorithm, again resulting in algorithm limited to a particular inverse problem and measurement matrix. Instead, in this paper the network is trained independently of $A$ and can be generalized across different inverse problems. This aspect is shared by two other neural-network-based solvers \cite{venkatakrishnan2013plug,rick2017one}, 
% Check this
however, they  model the image prior only implicitly by training a denoiser or a proximal map, and perhaps for this reason appear to require massive quantity of training samples. Importantly, very little is known about why and when they perform well, as even if the learned proximal map is assumed to be exact, there is no theoretical  convergence guarantee or bound on the recovery error.\\

In this work, we leverage the success of generative adversarial network (GAN) \cite{goodfellow2014generative,creswell2018generative,zhu2016generative,yeh2017semantic,berthelot2017began,ledig2017photo} in modeling the distribution of data. Indeed, GAN-based priors for natural images have been successfully employed to solve linear inverse  problems \cite{mardani2019deep,bora2017compressed,shah2018solving}. However, in \cite{mardani2019deep}, the operator $A$ is integrated into training the GAN, limiting it to a particular inverse problem. We therefore focus on the recent papers \cite{bora2017compressed,shah2018solving}  closest to our work, for extensive comparisons.
\\

Bora \etal \cite{bora2017compressed} do not have a guarantee on the convergence of their algorithm for solving the non-convex optimization problem, requiring several  random initializations. Similarly, in \cite{shah2018solving}, the inner loop uses a gradient descent algorithm to solve a non-convex optimization problem  with no guarantee of convergence to a global optimum. Furthermore, the conditions imposed in \cite{shah2018solving} on the random Gaussian measurement matrix for convergence of their outer iterative loop are unnecessarily stringent and cannot be achieved with a moderate number of measurements. Importantly, both these methods require expensive computation of the Jacobian $\nabla_z G$ of the differentiable generator $G$ with respect to the latent input $z$. Since computing  $\nabla_z G$ involves back-propagation through $G$ at every iteration, these reconstruction algorithms are computationally expensive and  even when implemented on a GPU they are slow.\\[1ex]
We propose a GAN-based projection network to solve compressed sensing recovery problems using projected gradient descent (PGD). We are able to reconstruct the image even with $61 \times$ compression ratio (\ie, with less than $1.6 \%$ of a full measurement set) using a random Gaussian measurement matrix. The proposed approach provides superior recovery accuracy over existing methods, simultaneously with a $60\text{-}80\times$ speed-up, making the algorithm useful for practical applications. We also provide theoretical results on the convergence of the reconstruction error, given that the measurement matrix $A$ satisfies certain conditions when restricted to the range $R(G)$ of the generator. We complement the theory by proposing a method to design a measurement matrix that satisfies these sufficient conditions for guaranteed convergence. We assess these sufficient conditions for both the random Gaussian measurement matrix and the designed matrix for a given dataset. Both our analysis and experiments show that with the designed matrix, $5\text{-}10\times$ fewer measurements suffice for robust recovery. Because the training of the GAN and projector is decoupled from the measurement operator, we demonstrate that other linear inverse problems like super-resolution and inpainting can also be solved using our algorithm without retraining.
\begin{figure}[tb]
    \centering
    \includegraphics[width=0.5\textwidth]{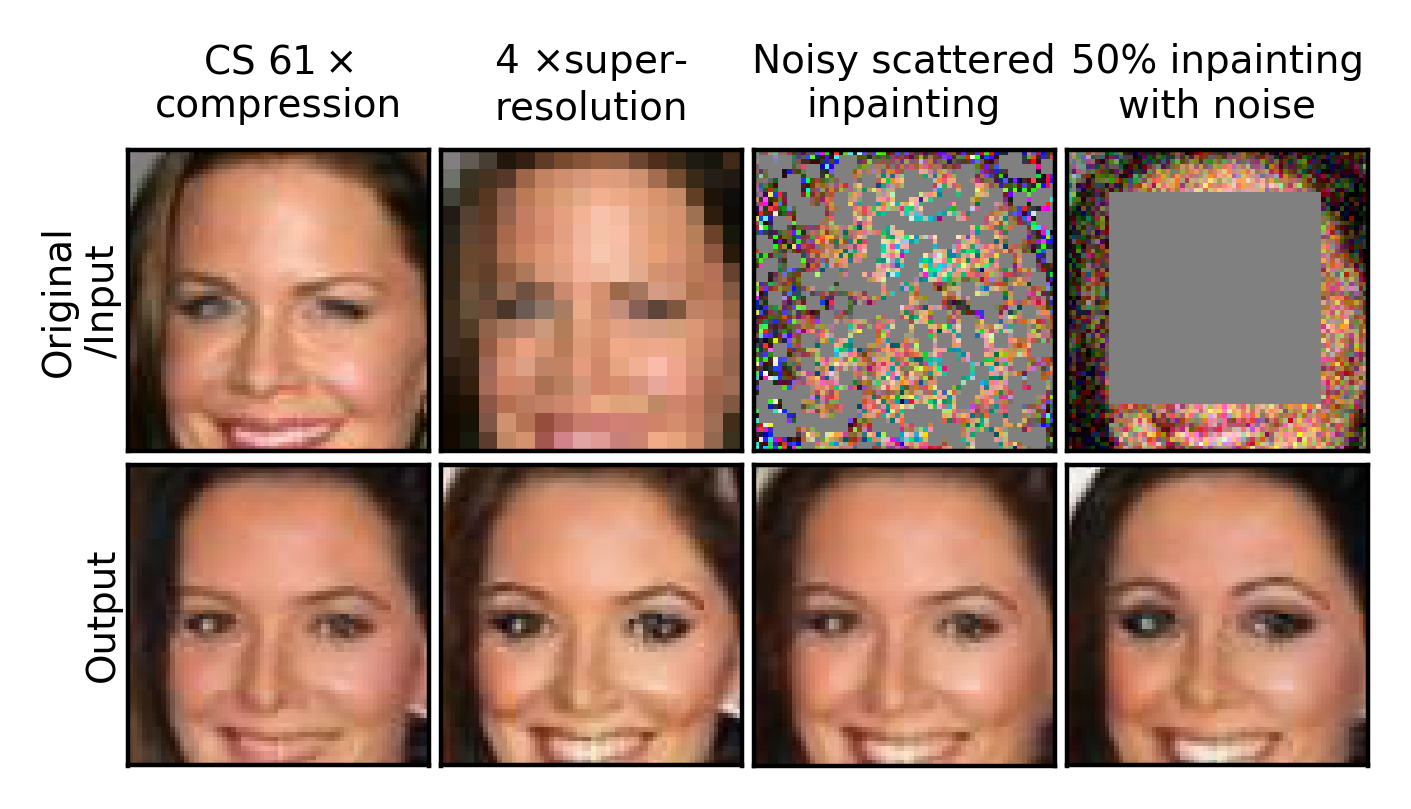}
    \caption{Our network-based PGD solves the following inverse problems: compressive sensing with 61$\times$ compression, $4\times$ super-resolution, scatterd inpaining with high noise ($\sigma=40$) and 50\% blocked inpainting with high noise.}
    \label{fig:first_fig}
    \vspace*{-1em}
\end{figure}

\section{Problem Formulation}
Let $x^*\in \mathbb{R}^{n} $ denote a ground truth image, $A$ a fixed measurement matrix, and $y = Ax^*+v \in \mathbb{R}^m$ the noisy measurement, with noise $v\sim\mathcal{N}(0,\sigma^2 I)$. We assume that the ground truth images lie in a non-convex set $S = R(G)$, the range of generator $G$. The maximum likelihood estimator (MLE) of $x^*$,  $\hat{x}_{MLE}$,  can be formulated as follows:
\begin{equation*}\label{eq:objective}
    \begin{aligned}
    \hat{x}_{MLE} &= \argmin_{x\in R(G)} -\log p(y|x)
    &=\argmin_{x\in R(G)} \|y-Ax\|_2^2 
    \end{aligned}
\end{equation*}
% To the best of our knowledge, we know of two papers which have solved \eqref{eq:objective} i.e. using GAN priors in inverse problems.
Bora \etal \cite{bora2017compressed} (whose algorithm we denote by CSGM) solve the optimization problem $\hat{z} = \argmin_{z\in \mathbb{R}^k} \|y-AG(z)\|^2 + \lambda \|z\|^2$
in the latent space ($z$), and set $\hat{x} = G(\hat{z})$. 
%They analyzed the behavior of minimizer of the loss function given that there exists one without explicitly providing the analysis of algorithm to find the minimizer. 
Their  gradient descent algorithm often gets stuck at local optima. Since the problem is non-convex, the reconstruction  is strongly dependent on the initialization of $z$ and requires several random initializations to converge to a good point. To resolve this problem, Shah and Hegde \cite{shah2018solving}  proposed a projected gradient descent (PGD)-based method (which we call PGD-GAN) to solve  \eqref{eq:objective}, shown in fig.\ref{fig:compare_pgds}(a). They  perform gradient descent in the ambient ($x$)-space and project the updated term onto $R(G)$. This projection involves solving another non-convex minimization problem (shown in the second box in 
fig.\ref{fig:compare_pgds}(a)) using the Adam optimizer \cite{kingma2014adam} for 100 iterations from a random initialization.
No convergence result is given for this iterative algorithm to perform the non-linear projection, and the  convergence analysis for the PGD-GAN algorithm \cite{shah2018solving} only holds if one assumes that the inner loop succeeds in finding the optimum projection. 

Our main idea in this paper is to replace this iterative scheme in the inner-loop with a learning-based approach, as it often performs better and does not fall into local optima \cite{zhu2016generative}. Another important benefit is that both earlier approaches require expensive computation of the Jacobian of $G$, which is eliminated in the proposed approach.
% This approach, Fig.\ref{fig:compare_pgds}(a) has inner-loop which takes a lot of time. 
%\\ (YL: better mention their approach has inner loop and takes long time)
\section{Proposed Method}
In this section, we introduce our methodology and architecture to train a projector using a pre-trained generator $G$ and how we use this projector to obtain the optimizer in (\ref{eq:objective}).
\begin{figure}[tb]
    \centering
    \begin{tabular}{c}
         \includegraphics[width=0.5\textwidth]{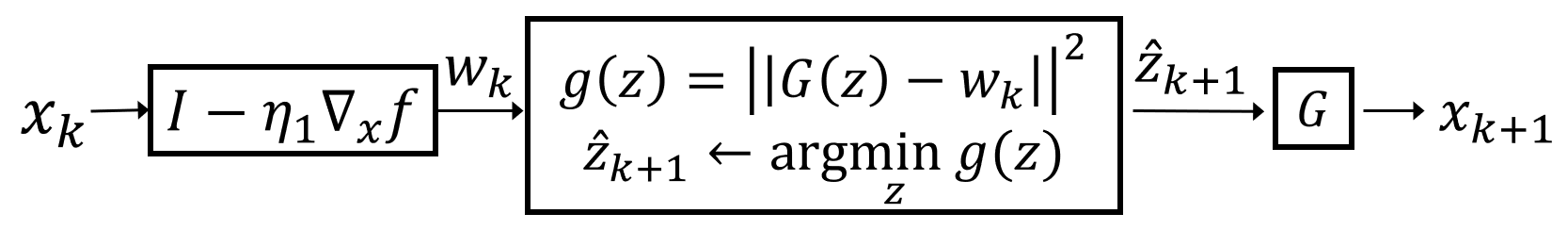}\\
          \small(a) PGD with inner-loop\\[2ex]
         \includegraphics[width=0.4\textwidth]{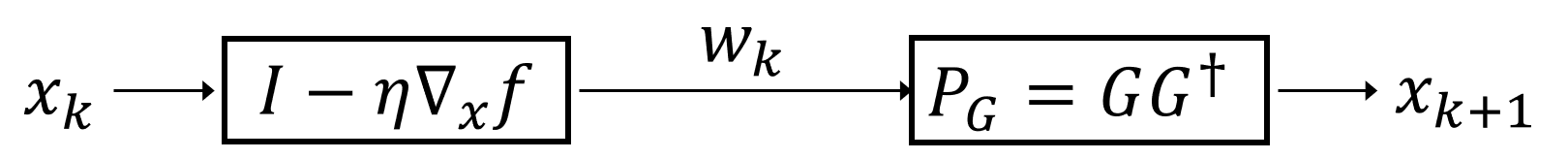}\\
          \small(b) Network-based PGD (NPGD)\\
    \end{tabular}
    \caption{(a) Block diagram of PGD using inner-loop \cite{shah2018solving}. $k$ represents the outer loop iterators and $\hat{z}_{k+1}$ is the optimizer of $\|G(z)-w_k\|^2$ obtained by solving the inner-loop using Adam optimizer. (b) Block diagram of our network-based PGD (NPGD) with $P_G = GG^{\dagger}$ as a network based projector onto $R(G)$. $f(x) = \|Ax-y\|^2$ is the cost function defined in {\protect \eqref{eq:objective}}}
    \label{fig:compare_pgds}
\end{figure}

\subsection{Inner-Loop-Free Scheme}
We show that by carefully designing a network architecture with a suitable training strategy, we can train a projector onto $R(G)$, the range of the generator $G$, thereby removing the inner-loop required in the earlier approach. The resulting iterative updates of our network-based PGD (NPGD) algorithm are shown in fig.\ref{fig:compare_pgds}(b). This approach eliminates the need to solve the non-convex optimization problem in the inner-loop, which depends on initialization and requires several restarts. Furthermore, our method provides a significant speed-up by a factor of $30\text{-}40\times$ on the CelebA dataset for two major reasons: (i) since there is no inner-loop, the total number of iterations required for convergence is significantly reduced, (ii) doesn't require computation of $\nabla G_{z}$ \ie the Jacobian of the generator with respect to the input, $z$. This expensive operation repeats back-propagation through the network for $T_{out} \times \#_{restarts}$(for \cite{bora2017compressed}) or $T_{out}\times T_{in}$ (for \cite{shah2018solving}) times, where $\#_{restarts}, T_{out} \text{ and } T_{in}$ are number of restarts, outer and inner iterations respectively. 
% Fig. \ref{fig:compare_pgds} shows block diagrams of the two schemes of PGD.

\subsection{Generator-based Projector}
\begin{figure}[tb]
    \centering
    \includegraphics[width=0.5\textwidth]{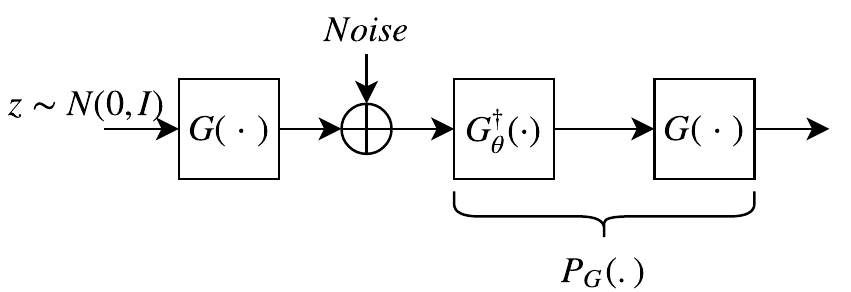}
    \caption{Architecture to train a projector onto range($G$)}
    \label{fig:proj_gan}
\end{figure}
A GAN consists of two networks, generator and discriminator, which follow an adversarial training strategy to learn the data distribution. A well-trained generator $G: \mathbb{R}^k\rightarrow R(G)\subset \mathbb{R}^n, k\ll n$ takes in a random latent variable $z\sim\mathcal{N}(0,I_k)$ and produces sharp looking images imitating the training data distribution in $\mathbb{R}^n$. The goal is to train a network that projects an image $x \in \mathbb{R}^n$ onto $R(G)$. The projector, $P_S$ onto a set $S$ should satisfy two main properties: $(i)$ \textit{Idempotence}, for any point $x$, $P_S(P_S(x)) = P_S(x)$, $(ii)$ \textit{Least distance}, for a point $\tilde{x}$, $P_S(\tilde{x}) = {\arg\min}_{x \in S} \|x - \tilde{x}\|^2$. Figure~\ref{fig:proj_gan} shows the network structure we used to train a projector using a GAN. We define the multi-task loss to be:
\begin{equation}\label{eq:invGobjective}
\footnotesize
\begin{aligned}
    \mathcal{L} ({\theta}) &= 
    \mathbb{E}_{z,\nu}\left[\left\lVert G\left(G^{\dagger}_{\theta}\left(G(z) + \nu \right)\right) - G(z)\right\rVert^{2}\right]\\ & \quad +\mathbb{E}_{z,\nu}\left[\lambda\left\|G^{\dagger}_{\theta}\left(G(z)+\nu\right) - z\right\|^2\right]
\end{aligned}
\end{equation}
where $G$ is a generator obtained from the GAN trained on a particular dataset. Operator $G^{\dagger}_{\theta}: \mathbb{R}^n\rightarrow \mathbb{R}^k$, parameterized by $\theta$, approximates a non-linear least squares pseudo-inverse of $G$ and $\nu  \sim \mathcal{N}(0, I_n)$ indicates noise added to the generator's output for different $z \sim \mathcal{N}(0, I_k)$ so that the projector network denoted by $P_G =GG^{\dagger}_{\theta}$ is trained on  points outside the range($G$) and learns to project them onto $R(G)$. The objective function consists of two parts. The first  is similar to standard \textit{Encoder-Decoder} framework, however, the loss function is minimized over $\theta$ -- the parameters of $G^{\dagger}$, while keeping the parameters of $G$ (obtained by standard GAN training) fixed. This ensures that  $R(G)$ doesn't change and $P_G = GG^{\dagger}$ is a mapping onto $R(G)$. The second part is used to keep $G^{\dagger}(G(z))$ close to true $z$ used to generate training image $G(z)$. This second term can be considered a regularizer for training the projector with $\lambda$ being the regularization constant.
\section{Theoretical Study}
\subsection{Convergence Analysis}
Let $f(x)=\|Ax-y\|_2^2$ denote the loss function of projected gradient descent. Algorithm \eqref{alg:1} describes the proposed network-based projected gradient descent (NPGD) to solve equation \eqref{eq:objective}.
\begin{algorithm}[t]
\caption{Network-based Projected Gradient Descent}
\label{alg:1}
\textbf{Input}: loss function $f$, $A, y, G, G^\dagger$\\
\textbf{Parameter}: step size $\eta( = \frac{1}{\beta})$\\
\textbf{Output}: an estimate $\hat{x}\in R(G)$
\begin{algorithmic}[1] %[1] enables line numbers
\STATE Let $t=0, x_0=A^Ty$.
\WHILE{$t<T$}
\STATE $w_t:= x_t -\eta A^T(Ax_t-y) $
\STATE $x_{t+1} := G(G^\dagger(w_t))$
\ENDWHILE
\STATE \textbf{return} $ \hat{x}=x_T $
\end{algorithmic}
\end{algorithm}

\begin{definition}[Restricted Eigenvalue Constraint (REC)] \label{def:rec}
Let $S\subset \mathbb{R}^n$. For some parameters $0<\alpha<\beta$, matrix $A\in\mathbb{R}^{m\times n} $ is said to satisfy the $REC(S,\alpha,\beta)$ if the following holds for all $x_1,x_2\in S$.   
\begin{equation}
    \alpha \|x_1-x_2\|^2 \leq \|A(x_1-x_2)\|^2 \leq \beta \|x_1-x_2\|^2. \label{eqn:rec}
\end{equation}
\end{definition}

\begin{definition}[Approximate Projection using GAN]
A concatenated network $G(G^\dagger(\cdot)): \mathbb{R}^n\rightarrow R(G)$ is a $\delta$-approximate projector, if the following holds for all $ x\in \mathbb{R}^n$:
\begin{equation}
    \|x-G(G^\dagger(x))\|^2 \leq \min_{z\in \mathbb{R}^k}\|x-G(z)\|^2+\delta
    \label{eqn:delta}
\end{equation}
\end{definition}
Theorem \ref{thm: conv}  provides upper bounds on the cost function and reconstruction error of our NPGD algorithm 
after $n$ iterations.
\begin{theorem} \label{thm: conv}
Let matrix $A\in\mathbb{R}^{m\times n} $  satisfy the $REC(S,\alpha,\beta)$ with $\beta/\alpha < 2$, and let the concatenated network $G(G^\dagger(\cdot))$ be a $\delta$-approximate projector. Then for every $x^*\in R(G)$ and  measurement $y=Ax^*$, executing algorithm \ref{alg:1} with step size $\eta=1/\beta$, will yield $f(x_{n})\leq (\frac{\beta}{\alpha}-1)^n f(x_0) + \frac{\beta\delta}{2 -\beta/\alpha}$. Furthermore, the algorithm achieves $\|x_n-x^*\|^2 \leq \big(C+\frac{1}{2\alpha/\beta-1}\big)\delta$ after $\frac{1}{2-\beta/\alpha}\log\big(\frac{f(x_0)}{C\alpha\delta}\big)$ steps. When $n\rightarrow\infty$, $\|x^*-x_\infty\|^2\leq \frac{\delta}{2\alpha/\beta-1}$. 
\end{theorem}
\begin{proof}
Please refer to the appendix.
\end{proof}

From theorem \ref{thm: conv}, one important factor is the ratio $\beta/\alpha$. This ratio largely determines the speed of linear ("geometric") convergence, as well as the reconstruction error $\|x^*-x_\infty\|^2$ at convergence. We would like $\beta/\alpha$ ratio as close to 1 as possible and must have $\beta/\alpha<2$ for convergence. 
It has been shown in \cite{baraniuk2009random} that a random matrix $A$ with orthonormal rows will satisfy this condition with high probability for $m$ roughly linear in dimension $k$ with log factors dependent on the properties of the manifold, in this case, $R(G)$. However, as we demonstrate later (see figure \ref{fig:MNIST-REC}), a random matrix often will not satisfy the desired condition $\beta/\alpha <2$ for small or moderate $m$. To extend into such regimes, we propose next a fast heuristic method to find a relatively good measurement matrix for an image set $S$, given a fixed $m$.

\subsection{Generator-based Measurement Matrix Design} \label{sec:prop_method-Amatrix}
There have been a few attempts to optimize the measurement matrix based on the specific data distribution. Hegde \etal \cite{hegde2015numax} find a deterministic measurement matrix that  satisfies $REC(S,1-\delta_S,1+\delta_S)$ for a given finite set $S$ of size $|S|$ , but their time complexity is $O(n^3+|S|^2n^2)$. Because the secant set $S$ (defined later) would be of cardinality $|S| =O(M^2)$ for a training set of size $M$, with $M \gg n$, the time complexity would be infeasible even for fairly small $n$-pixel images. Furthermore, the final number of required measurements $m$, which is determined by the algorithm,   depends on the isometry constant $\delta_S$, and cannot be specified in advance. Kvinge \etal \cite{kvinge2018gpu}  introduced a heuristic iterative algorithm to find a measurement matrix with orthonormal rows that satisfies the REC  with small $\beta/\alpha$ ratio, but their time complexity is $O\left(n^5\right)$ and the space complexity is $O(n^3)$, which is infeasible for a high-dimensional image dataset. Instead, our method, based on sampling from the secant set, has time complexity $O(Mn^2+n^3)$, and space complexity $O(n^2)$, where $M$ is a tiny fraction of $|S|$. 

\begin{definition}[Secant Set]
The normalized secant set of $G$ is defined as follows:
\begin{equation}
    \mathcal{S}(G) = \Big\{\frac{x_1-x_2}{\|x_1-x_2\|}: x_1,x_2\in R(G) \Big\}
\end{equation}
and the associated distribution is denoted as $\Pi_S$, where
\begin{equation}
    z_1,z_2\sim\mathcal{N}(0,I_k), s = \frac{G(z_1)-G(z_2)}{\|G(z_1)-G(z_2)\|}\sim \Pi_S
    \label{eq:def_s}
\end{equation}
\end{definition}
Given $\mathcal{S}(G)$, the optimization over A is as follows:
%\begin{equation}
\begin{align}
    &\min _ { A\in\mathbb{R}^{m\times n} }\frac{\beta}{\alpha} = \min _ { A\in\mathbb{R}^{m\times n} } \frac{\max _ {s\in \mathcal{S} ( G )}   \left\| A s\right\| ^ { 2 } }{\min _ {s\in \mathcal{S} ( G )}   \left\| A s\right\| ^ { 2 }} \label{eq:optimizeA_direct}\\ 
    &\leq \min _ { AA^T = I_m } \frac{1 }{\min _ {s\in \mathcal{S} ( G )}   \left\| A s\right\| ^ { 2 }} = \Big( \max_{AA^T = I_m}\min _ {s\in \mathcal{S} ( G )}   \left\| A s\right\| ^ { 2 } \Big)^{-1} \nonumber
\end{align}
%\end{equation}
The inequality is due to an additional constraint on $A: AA^T=I_m$. This results in the largest singular value of $A$ being 1 and hence the numerator term, $\max _ {s\in \mathcal{S} ( G )}   \left\| A s\right\| ^ { 2 }$, is at most 1. 
As the minimization in \eqref{eq:optimizeA_direct} 
%finding $\min _ {s\in \mathcal{S} ( G )}   \left\| A s\right\| ^ { 2 }$ 
requires iterating through the set $S$, we use the expected value over  $s \sim \Pi_S$ as a surrogate objective
%, $E_{s\sim P_S} \left[ \left\| A s\right\| ^ { 2 }\right]$ to replace this term:
\begin{equation}
    A = \argmax_{AA^T = I_m} E_{s\sim \Pi_S} \left[ \left\| A s\right\| ^ { 2 }\right] \approx \argmax_{AA^T = I_m}  \frac{1}{M} \sum_{j=1}^M \|As_j\|^2
    \label{eq:optimizeA}
\end{equation}
The last approximation replaces the surrogate objective by its empirical estimate obtained by sampling $M\gg n$ secants $(s_j)_{j=1}^M$ according to $\Pi_S$. 
% \begin{equation}\label{eq:empirical_max}
%     \max_{A} \frac{1}{M} \sum_{j=1}^M \|As_j\|^2 \text{ s.t. } AA^T=I_m
% \end{equation}
For $m$ and $M$ large enough, this designed measurement matrix would satisfy the condition $\beta/\alpha<2$ for most of the secants in $R(G)$. Constructing an $n\times M$ matrix $D = [s_1|s_2|\dots|s_M]$,  \eqref{eq:optimizeA} reduces to:
\begin{equation} \label{eq:Astar}
    A^* = \argmax_{A} \|AD\|_F^2  \text{ s.t. } AA^T=I_m
\end{equation}
The optimal $A^*$ in \eqref{eq:Astar} has rows equal to the $m$ leading eigenvectors
%from these $M$ samples would satisfy $A^{*T}=U_m$, where $DD^T=U\Lambda U^T$ is the eigenvalue decomposition (EVD) of 
$DD^T$. 
%and $U_m$ is the sub-matrix consisting of first $m$ columns in $U$. 
We compute $DD^T = \sum_{j=1}^M s_j s_j^T$ and its eigenvalue decomposition at time complexity $O(Mn^2+n^3)$ and space complexity $O(n^2)$.

Our approach to the design of $A$ is related to one of the steps described by \cite{kvinge2018gpu}, however by using the sampling-based estimates per \eqref{eq:def_s} and \eqref{eq:optimizeA} rather than the secant set for the entire training set, we reduce the computational cost by orders of magnitude to a modest level. 
\subsubsection{REC Histogram for $A$}
% \textbf{\color{blue} Modified by YL}
We analyze the $REC$ conditions by plotting the histogram of $\|As\|$ values for different measurement matrices $A \in R^{m\times n}$ in figure \ref{fig:MNIST-REC} where $s \in S$,  the secant set of the samples from $G$ trained on MNIST dataset. The left column shows the histograms for the random and $G$-based designed matrix. For random $A$, the spread of $\|As\|$ is clearly wider for few measurements $m$, resulting in $\beta/\alpha \not < 2$. For the designed $A$, the histogram is more concentrated. Even with as few as  $m=20$ measurements (for MNIST), the designed $A$ satisfies the sufficient condition $\beta/\alpha < 2$  for convergence of the PGD algorithm, thus ensuring  stable recovery. The middle columns shows the histograms corresponding to the downsampling $A$ that takes the spatial averages of $f \times f$, $f=2,3,4,5$, pixel values to generate low-resolution images. The right column shows the histograms for the inpainting $A$ that masks out a centered square of various sizes. As expected, with more difficult recovery problems the spread increases. However, for each inverse problem (defined by a matrix $A$), the ratio $\beta/\alpha$ can be estimated for \eg, 99.9\% of the samples, providing, in combination with Theorem~\ref{thm: conv}, an explicit quantitative guarantee. 
\begin{figure}[tb]
    \centering
    \includegraphics[width=0.5\textwidth]{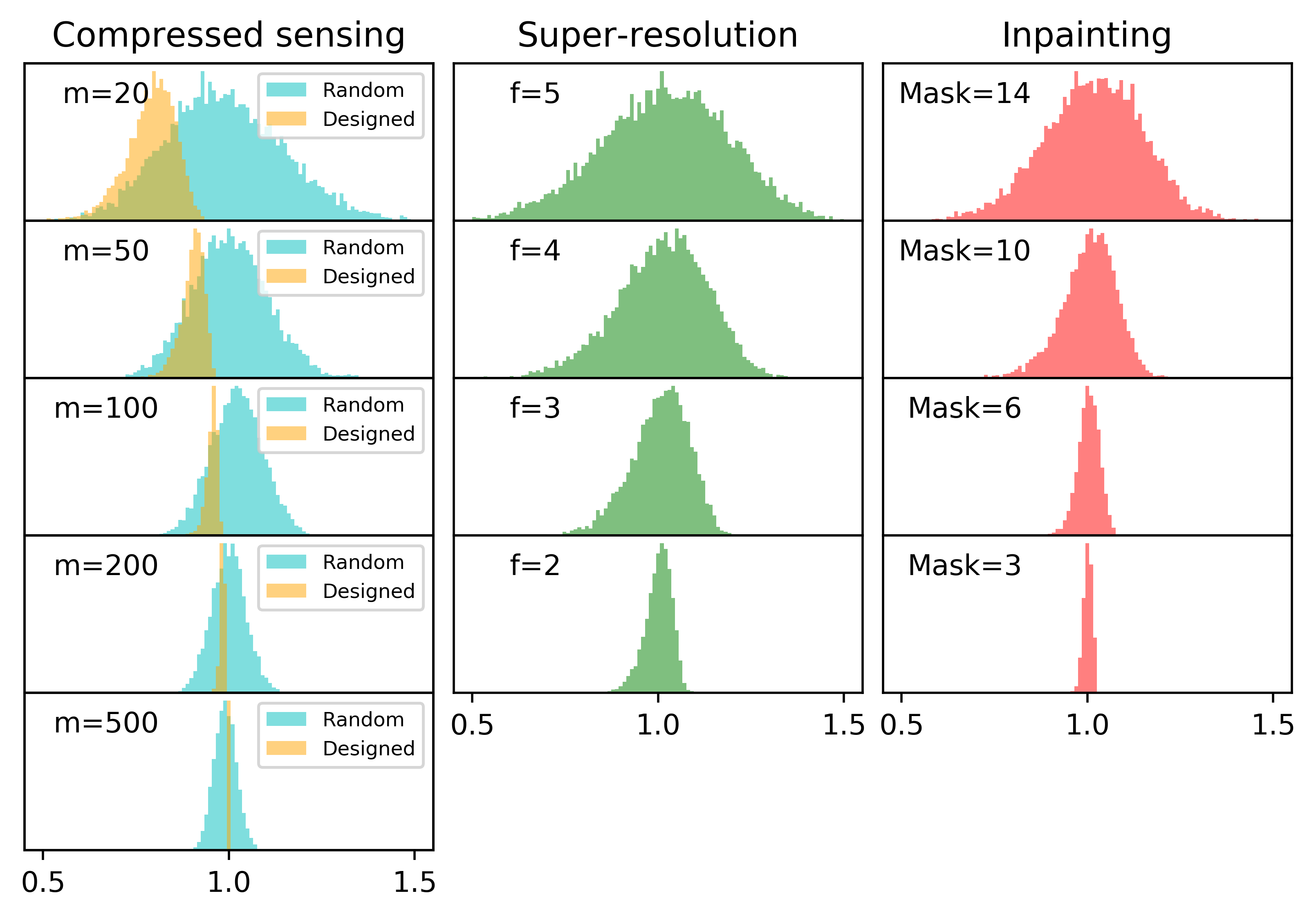}
    \caption{Distribution of $\|As\|$ with different $A$. Left: Random (cyan) and Designed matrix (orange) with different $m$. Middle: Downsampling matrix (green) with different $f$. Right: Inpainting matrix (red) with different mask size.}
    \label{fig:MNIST-REC}
\end{figure}

\section{Experiments}
\textbf{Network Architecture:} We implement two GAN architectures: $(i)$ Deep convolutional GAN (DCGAN) \cite{radford2015unsupervised} for MNIST and CelebA, $(ii)$ Self-attention GAN (SAGAN) \cite{zhang2018self} for LSUN church-outdoor dataset. DCGAN builds on multiple convolution, transpose convolution, and ReLU layers, and uses batch normalization and dropout for better generalization, whereas SAGAN combines convolutions with self-attention mechanisms in both, generator and discriminator, allowing for long-range dependency modeling to generate images with high-resolution details. For DCGAN, we have used standard objective function of the adversarial loss, whereas for SAGAN, we minimized the hinge version of the adversarial loss \cite{miyato2018spectral}. 
%For the datasets, MNIST and CelebA, we design  projection networks of similar structure but different number of layers. 
The architecture of the model $G^{\dagger}$ is similar to that of the discriminator $D$ in the GAN and only differs in the final layer, where we add a fully-connected layer with output size same as the latent variables dimension $k$. For training $G^{\dagger}$, we used the architecture shown in Fig. \ref{fig:proj_gan} and objective defined in \eqref{eq:invGobjective}, while keeping the pre-trained $G$ fixed. We found that using $\lambda = 0.1$, in \eqref{eq:invGobjective}, gave the best performance. The noise $\nu$ used for perturbing the training images $G(z)$ follows $\mathcal{N}(0,\sigma^2 I)$. We observed that training with low $\sigma$ results in a projector similar to an identity operator and hence only projecting close-by points onto $R(G)$, whereas for large $\sigma$ the projector violates idempotence. We empirically set $\sigma=1$. We then obtain a projection network $P_G = GG^{\dagger}$ that approximately projects images lying outside $R(G)$ onto $R(G)$. We empirically pick latent variable dimension $k=100$.\\
\begin{figure}
    \centering
    \includegraphics[width=0.45\textwidth]{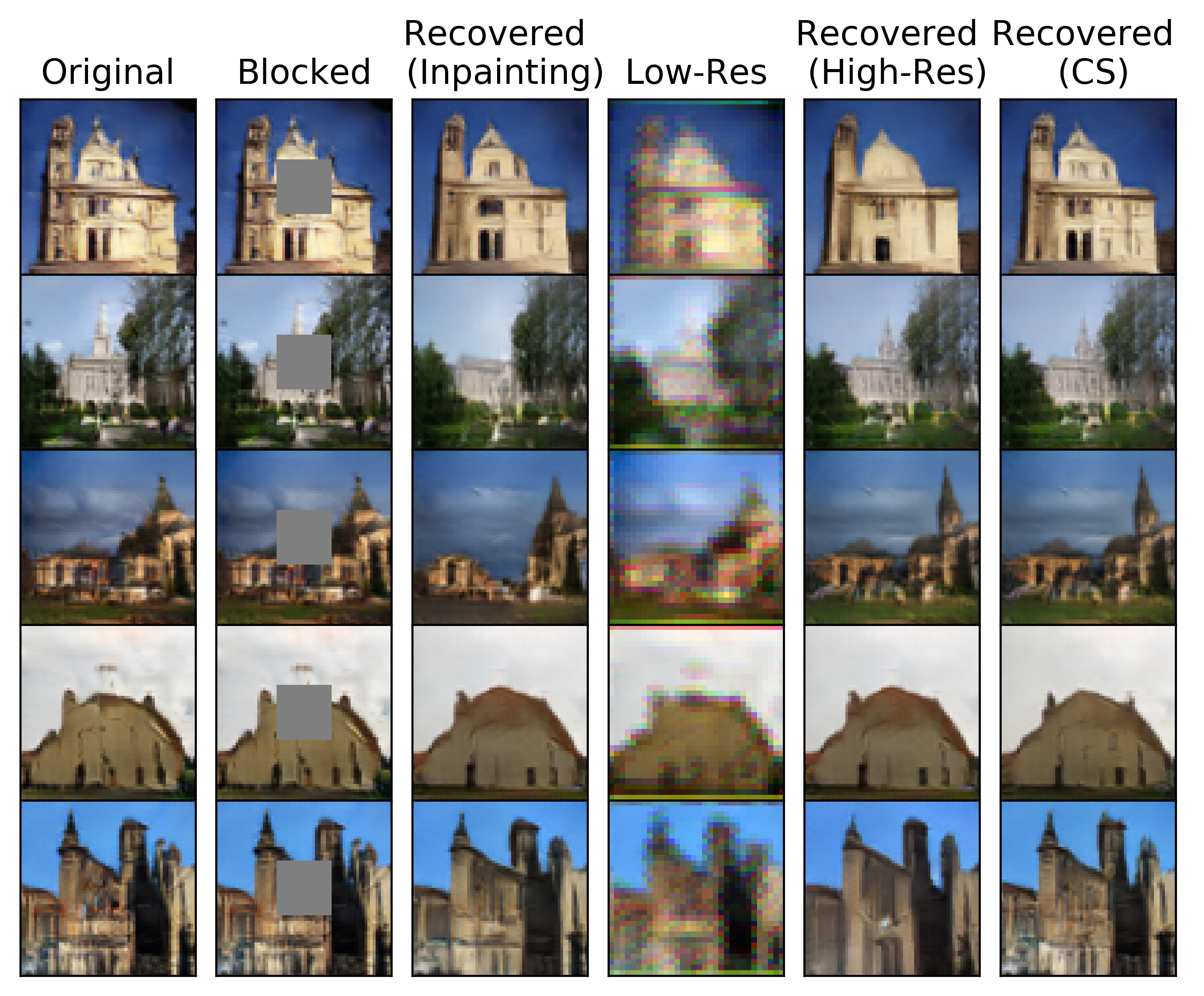}
    \caption{Recovery of LSUN church-outdoor images in inpainting (mask size = $20$), super-resolution ($4\times$) and Compressed Sensing (CS, $m=1000$) tasks.}
    \label{fig:church_result}
\end{figure}
\textbf{MNIST} dataset \cite{lecun1998gradient} consists of $28\times28$ greyscale images of digits with $50,000$ training and $10,000$ test samples. We pre-train the GAN consisting of $4$ transposed convolution layers for $G$ and $4$ convolution layers in the discriminator $D$ using rescaled images lying between $[-1, 1]$. We use $z\sim\mathcal{N}(0, I_{k})$ as the $G$'s input. The GAN is trained using the Adam optimizer with learning rate $0.0001$, mini-batch size of $128$ for $40$ epochs. For training the pseudo-inverse of $G$ \ie $G^{\dagger}$, we minimize the objective (\ref{eq:invGobjective}), using samples generated from $G(z)$, and with the same hyper-parameters used for the GAN.\\
\textbf{CelebA} dataset \cite{liu2015faceattributes} consists of more than $200,000$ celebrity images. We use the aligned and cropped version, which preprocesses each image to a size of $64\times64\times 3$ and scaled between $[-1, 1]$. We randomly pick $160,000$ images for training the GAN. Images from the $40,000$ held-out set are used for evaluation. The GAN consists of $5$ transposed convolution layers in the $G$ and $5$ convolution layers in $D$. GAN is trained for $35$ epochs using Adam optimizer with learning rate $0.00015$ and mini-batch size $128$. $G^{\dagger}$ is trained in the same way as for the MNIST dataset. \\
\textbf{LSUN} church-outdoor dataset \cite{yu2015lsun} consists of more than $126,000$ cropped and aligned images of size $64 \times 64 \times 3$ scaled between $[-1, 1]$. DCGAN generates high-resolution details using spatially local points in lower-resolution feature maps, whereas in SAGAN, details can be generated using information from many feature locations making it a natural choice for diverse dataset such as LSUN. The SAGAN consists of $4$ transposed convolution layers and $2$ self-attention modules at different scales in $G$ and $4$ convolution layers and $2$ self-attention modules in $D$. Each self-attention module consists of 3 convolution layers and are added at the $3rd$ and $4th$ layers of the two networks. SAGAN uses conditional batch normalization in $G$ and projection in $D$. Spectral normalization is used for the layers in both $G$ and $D$. We use ADAM optimizer with $\beta_1 = 0$  and $\beta_2=0.9$, learning rate $0.0001$ and mini-batch size $64$ for the GAN training. $G^\dagger$, consisting of self-attention mechanism similar to $D$, is trained using the objective \ref{eq:invGobjective} using the ADAM optimizer with $\beta_1=0.9$ and $\beta_2=0.999$, learning rate $0.001$ and mini-batch size of $64$ for $100$ epochs.

We compare the performance of our algorithm on MNIST and CelebA with other GAN-prior solvers (\cite{bora2017compressed,shah2018solving}) and sparsity-based methods, Lasso with discrete cosine transform (DCT) basis \cite{tibshirani1996regression} and total variation minimization method (TVAL3) \cite{li2009user} for linear inverse problems, namely compressed sensing (CS), super-resolution and inpainting. For CS, we extensively evaluate the reconstruction performance with the random Gaussian and designed measurement matrices. Furthermore, we demonstrate the recovery of LSUN church-outdoor dataset images using the proposed method for the different problems in Fig. \ref{fig:church_result}.
\subsection{Compressed Sensing}
\subsubsection{Recovery with random Gaussian matrix}
In this set-up, we use the same measurement matrix $A$ as (\cite{bora2017compressed,shah2018solving}) \ie $A_{i, j} \sim N(0, 1/m)$ where $m$ is the number of measurements. For MNIST, the measurement matrix $A \in R^{m\times784}$, with $m = 20, 50, 100 ,200$, whereas for CelebA, $A\in R^{m\times12288}$, with $m=200, 500, 1000, 2000$. Figure \ref{fig:mnist_random} shows the recovery results for MNIST images from the test set. Our NPGD algorithm  performs better than others and avoids local optima. Figure \ref{fig:celebA_random} shows the reconstruction of eight test images from CelebA. Our algorithm outperforms the other three methods visually as it is able to preserve detailed facial features such as sunglasses, hair and has accurate color tones. Figures \ref{fig:recovery error a} and \ref{fig:recovery error c} provide a quantitative comparison for MNIST and CelebA, respectively.
\begin{figure}[tb]
    \centering
    \includegraphics[width=0.45\textwidth]{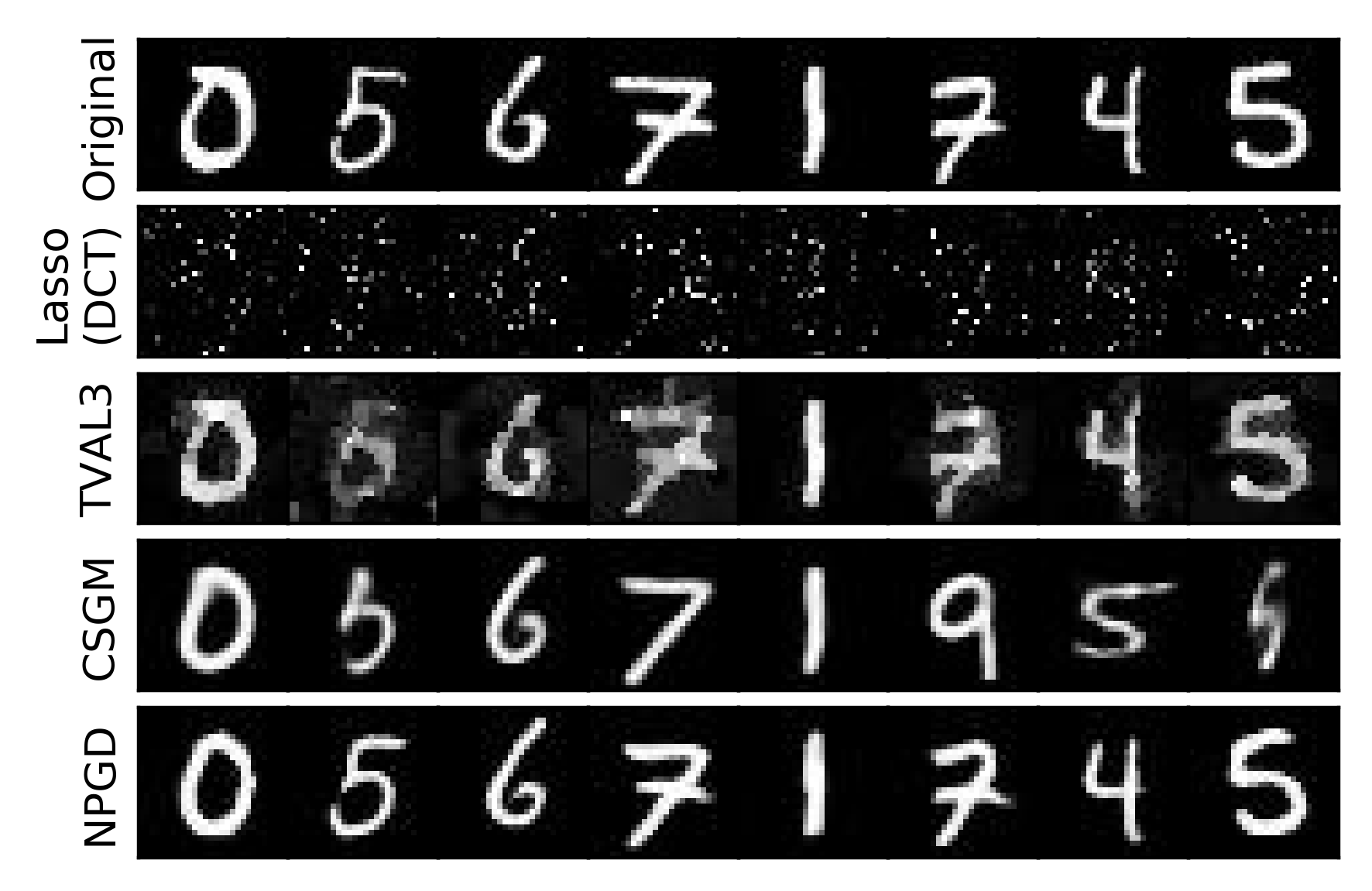}
    \caption[]{Reconstruction using Gaussian matrix with $m=100$. 
    \footnotemark}
    \label{fig:mnist_random}
\end{figure}
\footnotetext{Code of Shah \etal (PGD-GAN) for MNIST not available}

\vspace{-1em}
\subsubsection{Recovery with the designed matrix}
In this set-up, we use the $G$-based designed $A$ described in the section \ref{sec:prop_method-Amatrix}.
We observe that recovery with the designed $A$ is possible for much fewer measurements $m$. This corroborates our assessment based on Figure \ref{fig:MNIST-REC} that the designed matrix satisfies the desired REC condition with high probability for most of the secants, for smaller $m$. Figures \ref{fig:recovery error a}, \ref{fig:recovery error c} show that our algorithm consistently outperforms other approaches in terms of reconstruction error and structural similarity index (SSIM) for a random $A$. Furthermore, with the designed $A$, we are able to get performance on-par with the random matrix using $5\text{-}10\times$ smaller $m$. Figures \ref{fig:recovery error b},\ref{fig:recovery error d} show the recovered images with the designed and a random $A$ using our algorithm for different $m$. Clearly, recovery with the random $A$ requires much bigger $m$ than the designed one to achieve similar performance.
\begin{figure}[tb]
    \centering
    \includegraphics[width=0.42\textwidth]{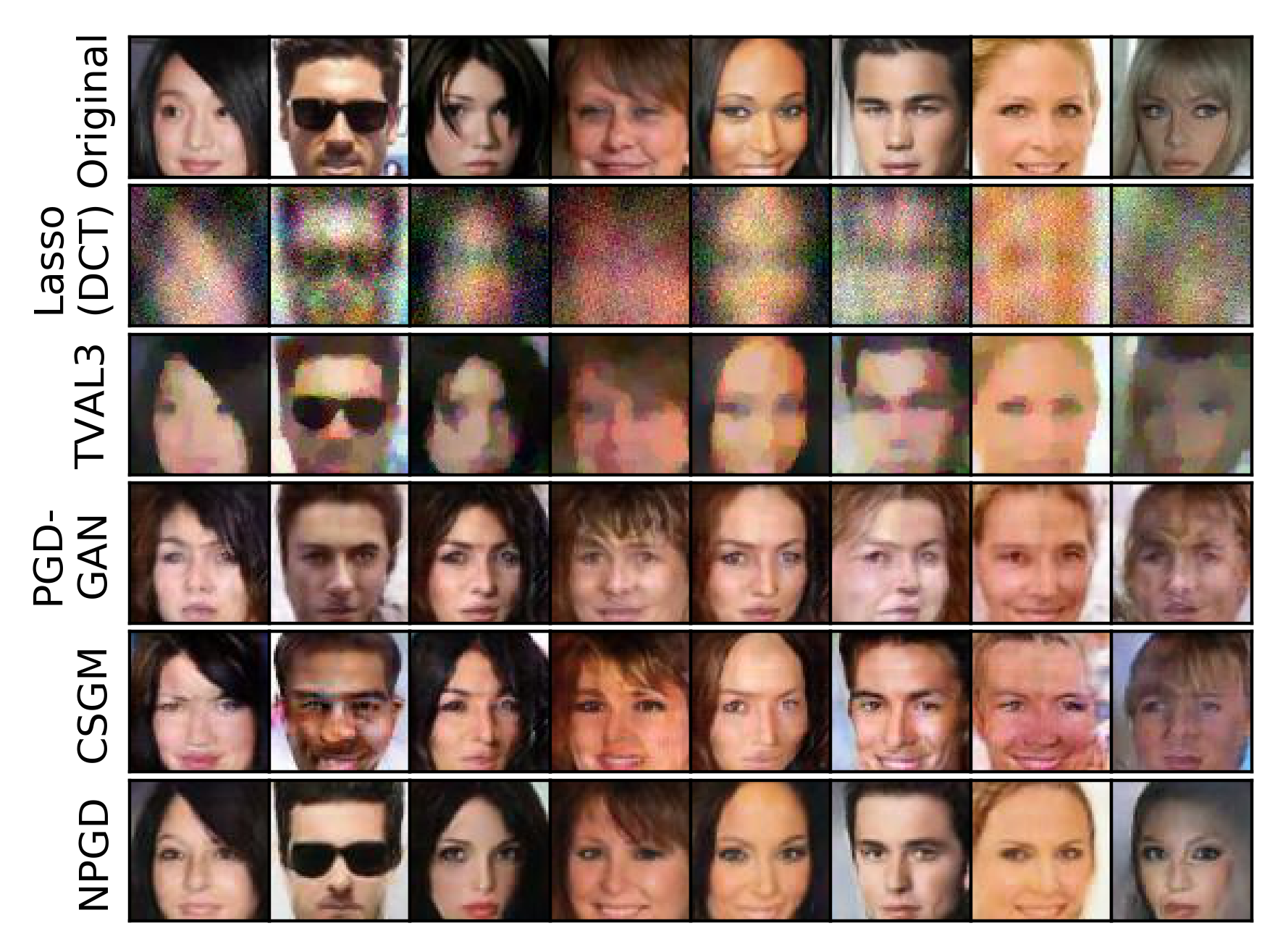}
    \caption{Reconstruction using Gaussian matrix with $m=1000$.}
    \label{fig:celebA_random}
\end{figure}
\begin{figure*} [tb]
    \vspace{-1em}
    \centering
    \begin{subfigure}[b]{.63\linewidth}
    \centering
    \includegraphics[width=.99\textwidth]{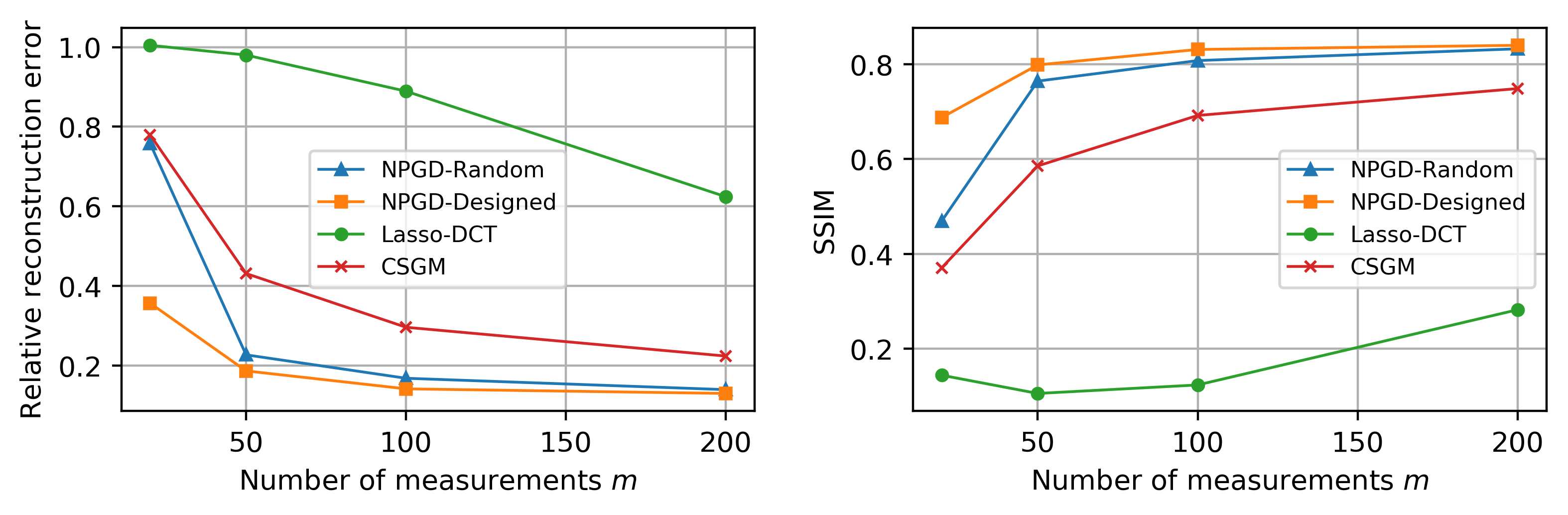}
    \caption{}\label{fig:recovery error a}
    \end{subfigure}
    \begin{subfigure}[b]{.25\linewidth}
    \centering
    \includegraphics[width=.99\textwidth]{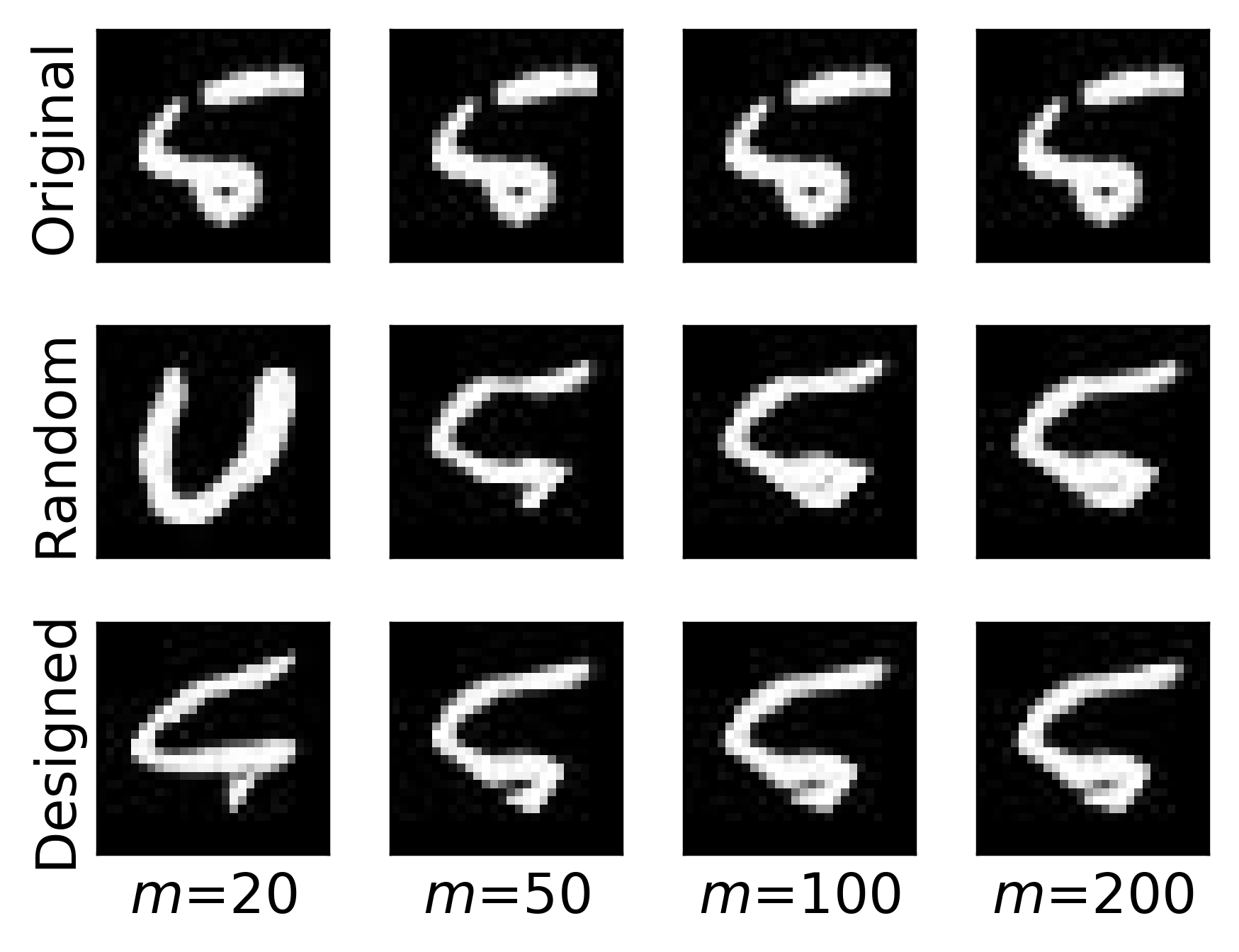}
    \caption{}\label{fig:recovery error b}
    \end{subfigure}\\    
    \begin{subfigure}[b]{.63\linewidth}
    \centering
    \includegraphics[width=.99\textwidth]{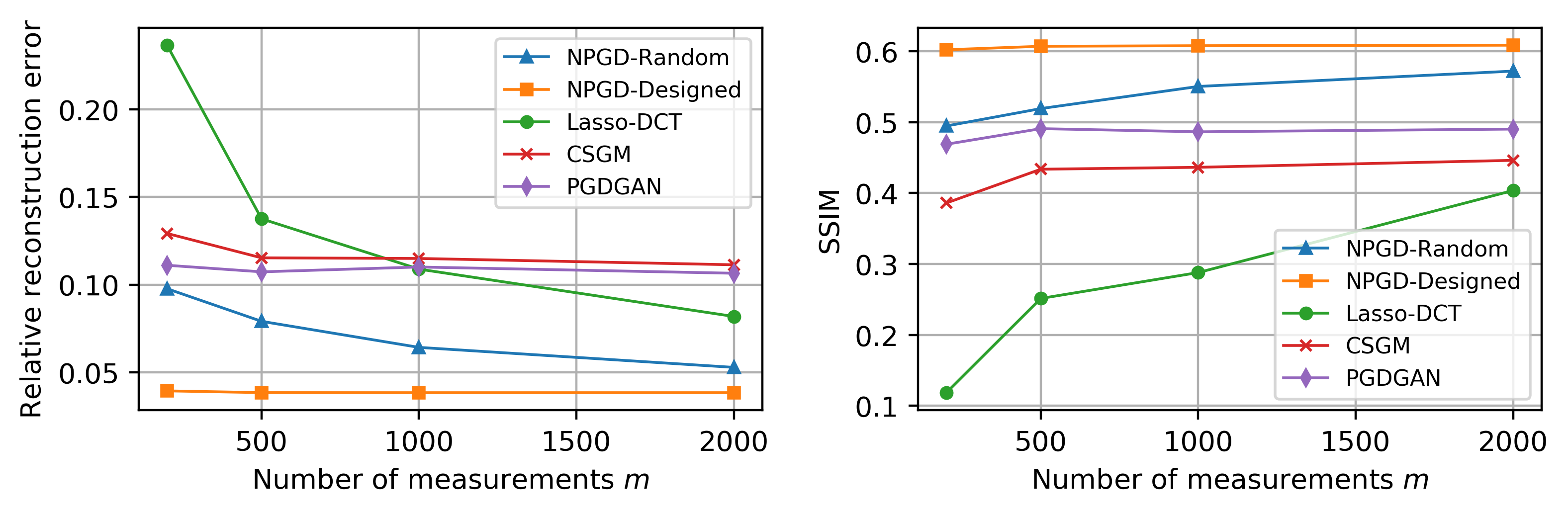}
    \caption{}\label{fig:recovery error c}
    \end{subfigure}
    \begin{subfigure}[b]{.25\linewidth}
    \centering
    \includegraphics[width=.99\textwidth]{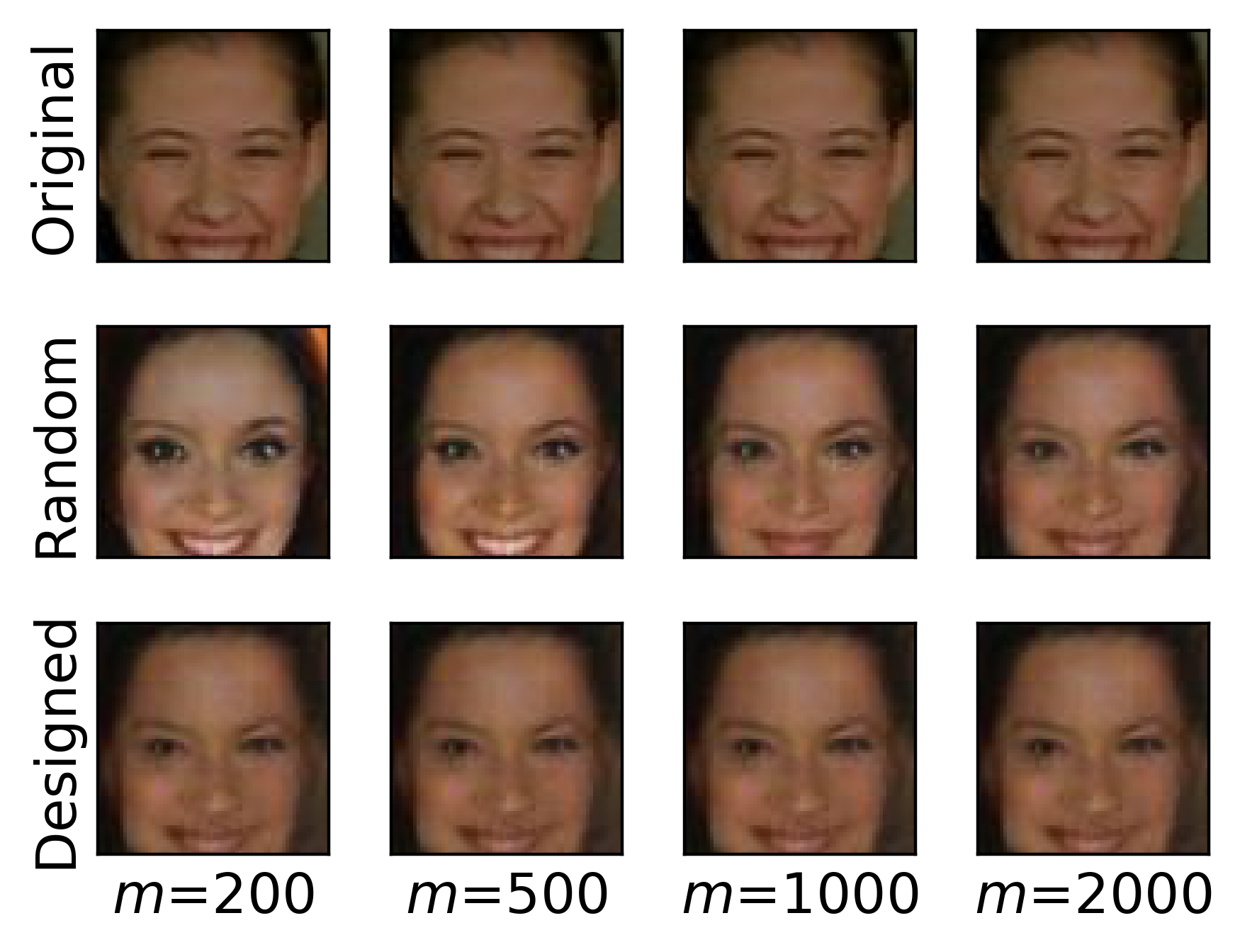}
    \caption{}\label{fig:recovery error d}
    \end{subfigure}\\  
    \caption{(a) Relative error $\|x^*-\hat{x}\|^2/\|x^*\|^2$ and SSIM of  reconstruction algorithms for MNIST dataset with $m = 20, 50, 100, 200$ measurements. (b) MNIST reconstructions  with a random Gaussian (middle row) and the designed matrix with orthonormal rows based on $G$ (bottom row) using different $m$. (c) Relative error and SSIM for CelebA dataset with  $m=200, 500, 1000, 2000$ measurements. (d) CelebA reconstructions, as in (b).
    %images from CelebA dataset with random Gaussian (middle row) and a designed matrix with orthonormal rows (bottom row).
    }
    \label{fig:recovery error}
\end{figure*}
\subsection{Super-resolution}
Super-resolution refers to recovering the high-resolution image from a single low-resolution image, often modeled as a blurred and downsampled image of the original. This super-resolution problem is just a special case in our framework of linear measurements.
%: the measurement matrix $A$ corresponds to image downsampling by obtaining local spatial averages. 
We simulate the blurring+downsampling by taking the spatial averages of $f\times f$ pixel values (in RGB color space), where $f$ is the ratio of downsampling. This corresponds to blurring by an $f\times f$ box impulse response, followed by downsampling.  We test our algorithm with $f=2,3,4$, corresponding to $4\times$, $9\times$ and $16\times$-smaller image sizes, respectively. We note that for higher $f$,  the measurement matrix $A$ may not satisfy the desired $REC(S, \alpha, \beta)$ with $\frac{\beta}{\alpha} < 2$ (see figure \ref{fig:MNIST-REC}) required for convergence of our algorithm and, consequently, our theorem might not be applicable. Results for MNIST in figure \ref{fig:mnist_sr_a}-\ref{fig:mnist_sr_c} shows that recovery performance indeed degrades with increasing $f$, however, our NPGD algorithm, gives better reconstructions than Bora \etal \cite{bora2017compressed}.
% We use our recovery algorithm for different $A's$ corresponding to super-resolution with different factors.
\begin{figure}[tb]
    \centering
    \begin{subfigure}[b]{.323\linewidth}
    \centering
    \includegraphics[width=.99\textwidth]{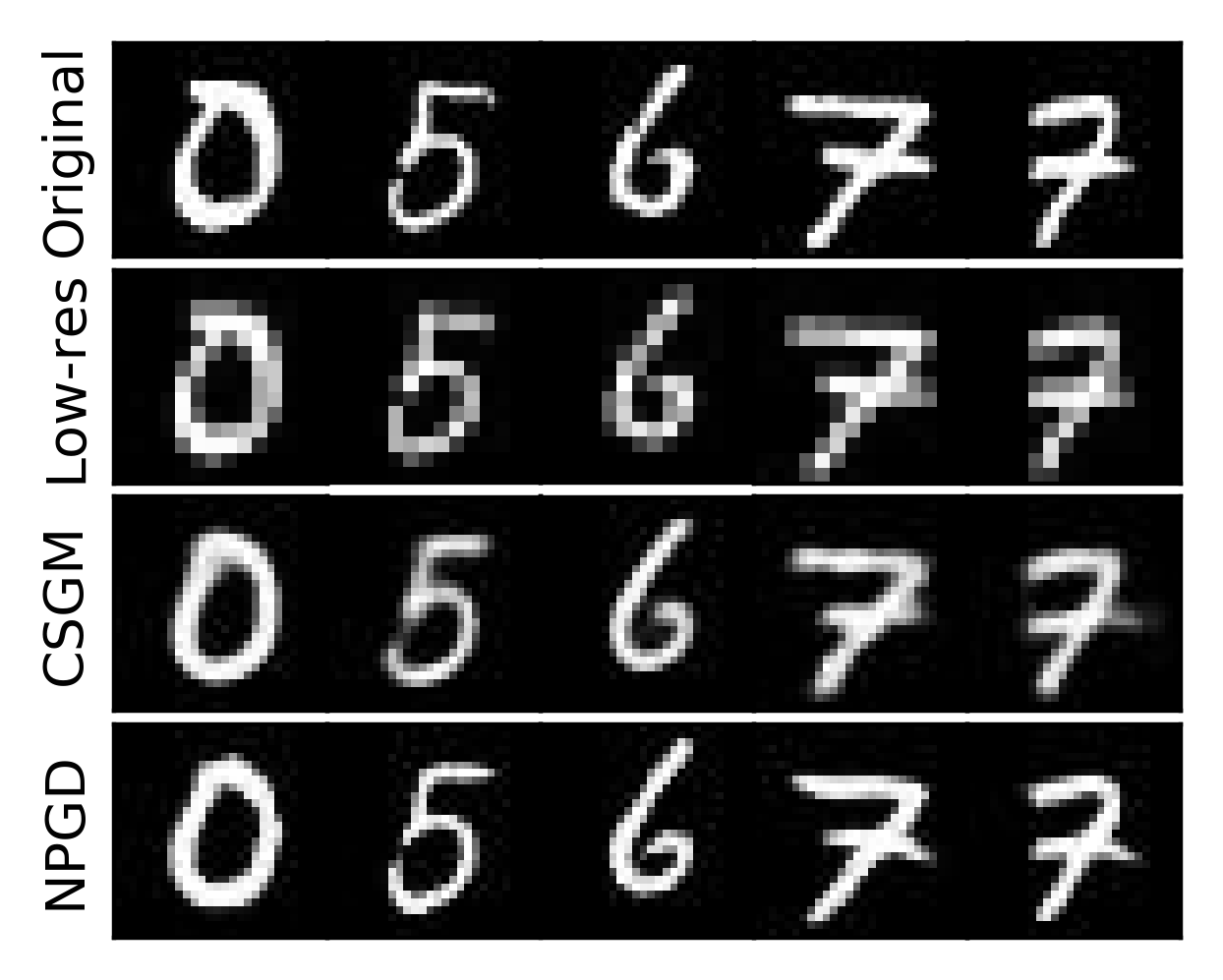}
    \caption{$4\times$ low-res}\label{fig:mnist_sr_a}
    \end{subfigure}
    \begin{subfigure}[b]{.323\linewidth}
    \centering
    \includegraphics[width=.99\textwidth]{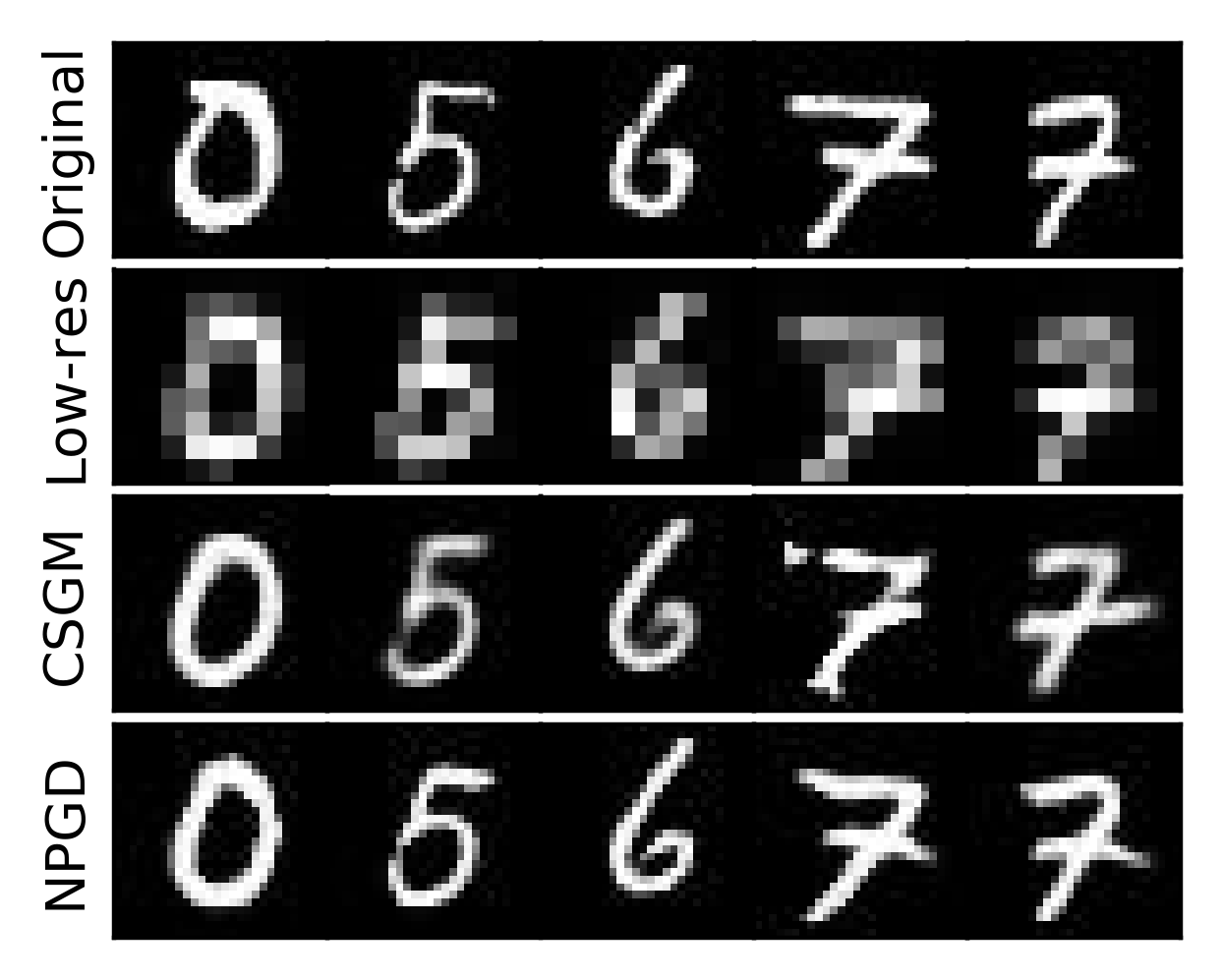}
    \caption{$9\times$ low-res}\label{fig:mnist_sr_b}
    \end{subfigure}
    \begin{subfigure}[b]{.323\linewidth}
    \centering
    \includegraphics[width=.99\textwidth]{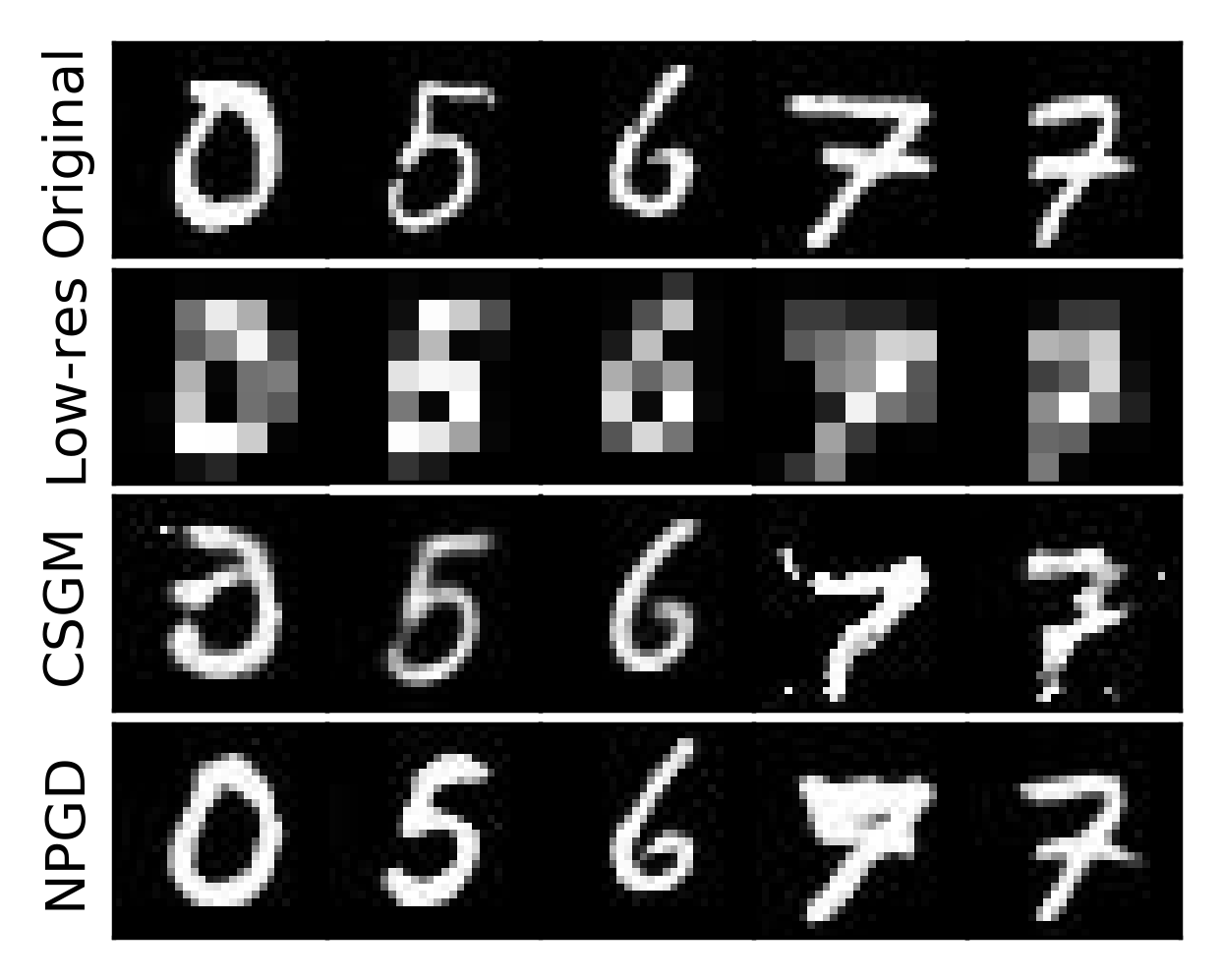}
    \caption{$16\times$ low-res}\label{fig:mnist_sr_c}
    \end{subfigure}
    \caption{Super-resolution on MNIST dataset. Row 1: original image $x$. Row 2: low-resolution images $y$, upsampled using constant padding, Row 3: high resolution image recovered by \cite{bora2017compressed}. Row 4: high-resolution image recovered by our method.}
    \label{fig:mnist_super_resolution_inpainting}
\end{figure}
\vspace{-1em}
\begin{figure}[tb]
    \centering
    \begin{subfigure}[b]{.323\linewidth}
    \centering
    \includegraphics[width=.99\textwidth]{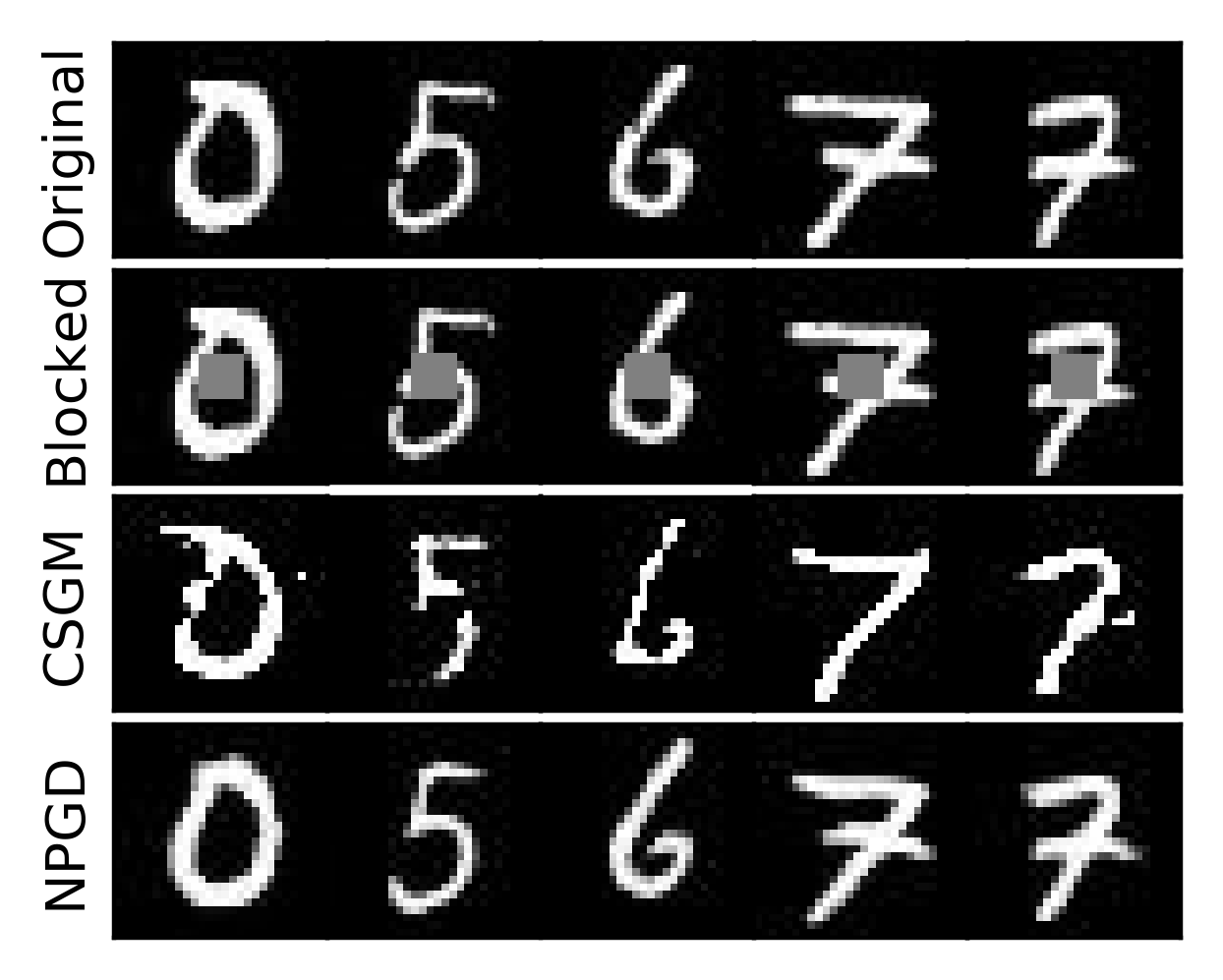}
    \caption{Mask size = 6}\label{fig:mnist_inpaint_d}
    \end{subfigure}
    \begin{subfigure}[b]{.323\linewidth}
    \centering
    \includegraphics[width=.99\textwidth]{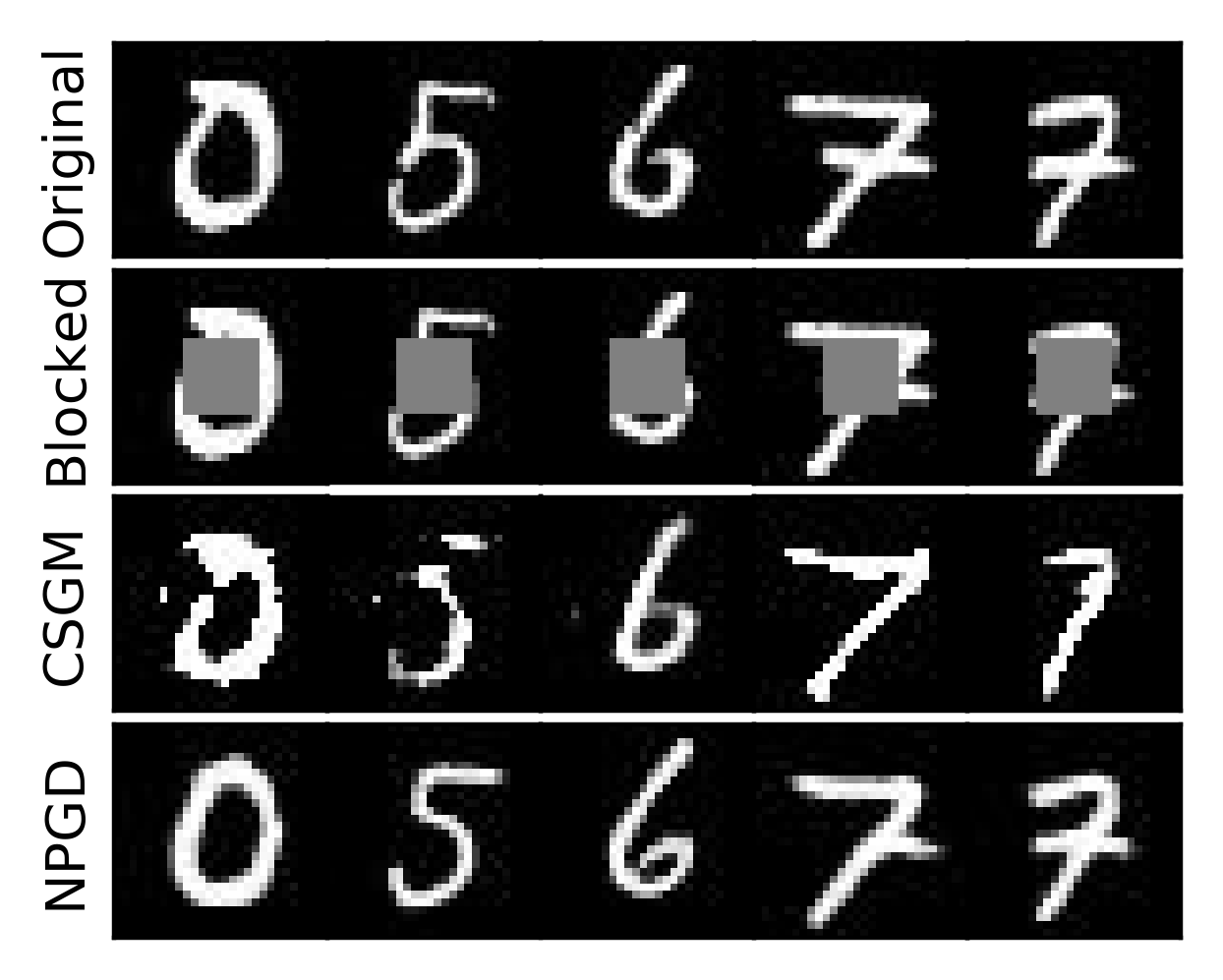}
    \caption{Mask size = 10}\label{fig:mnist_inpaint_e}
    \end{subfigure}
    \begin{subfigure}[b]{.323\linewidth}
    \centering
    \includegraphics[width=.99\textwidth]{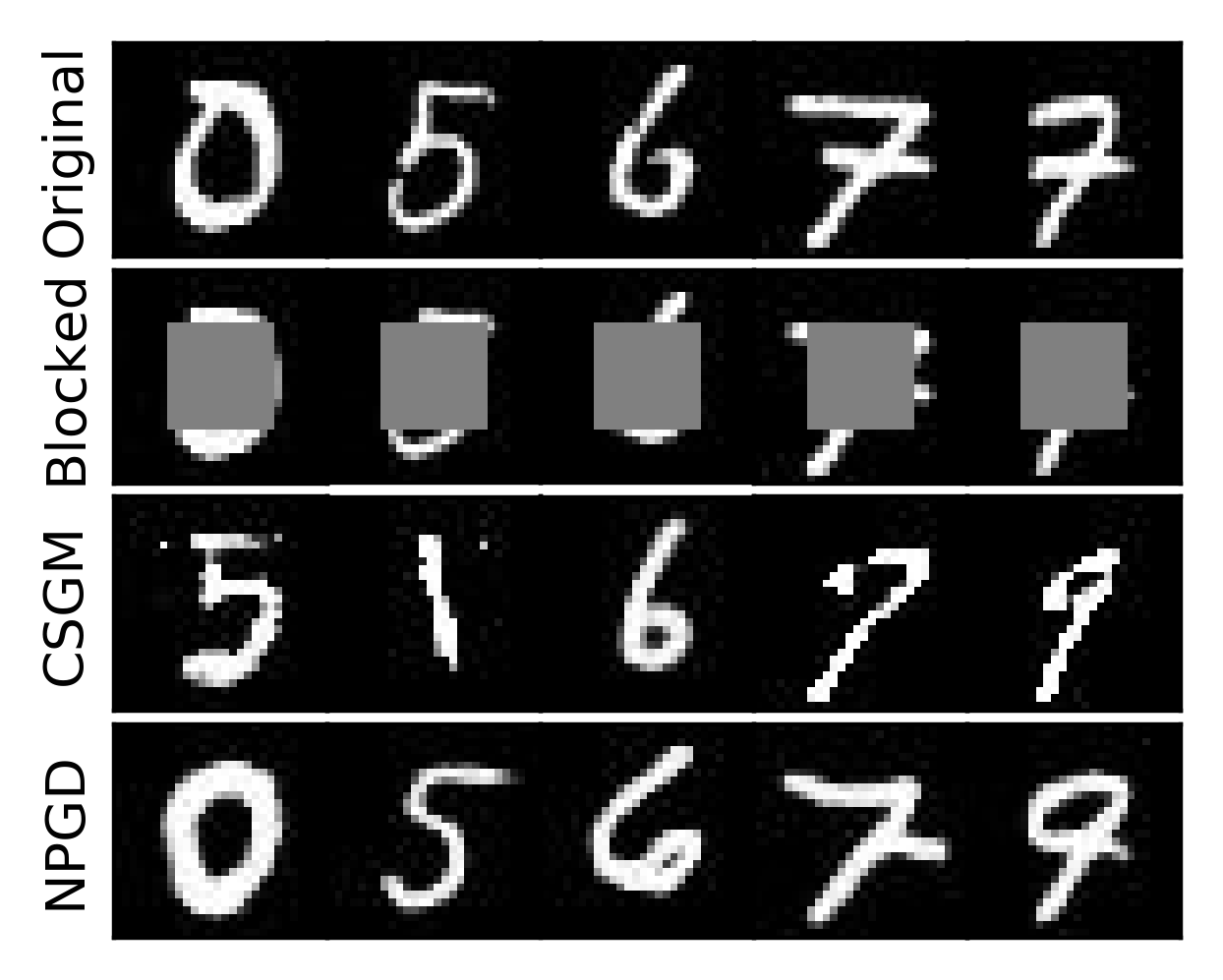}
    \caption{Mask size = 14}\label{fig:mnist_inpaint_f}
    \end{subfigure}
    \caption{Inpainting in MNIST dataset. Row 1: original image $x$. Row 2: image $y$ with center block missing. Row 3: image recovered by \cite{bora2017compressed}. Row 4: image recovered by our method.}
\end{figure}
\subsection{Inpainting}
Inpainting refers to recovering the entire image from a partly occluded version. In this case, $y$ is an image with masked
regions and $A$ is the linear operation applying a pixel-wise mask to the original image $x$. Again, this is a special case of linear measurements where each measurement corresponds to an observed pixel. For experiments on the MNIST dataset, we apply a centered square mask of size $6, 10, 14$. Recovery results in figure \ref{fig:mnist_inpaint_d}-\ref{fig:mnist_inpaint_f} show that our method  consistently outperforms \cite{bora2017compressed} and recovers almost perfectly for mask-size less than $10$. The results align with the $REC$ histogram for inpainting (figure \ref{fig:MNIST-REC}), which shows that for higher mask-size, the desired $REC$ condition for guaranteed convergence may not be satisfied.
% The results align with the mask-size required to satisfy the desired $REC$ condition for guaranteed convergence (see figure \ref{fig:MNIST-REC}).
\subsection{Comparison of Run-time for Recovery}
Table~\ref{tab:comp_exec_time_rec} compares the run times  of our network-based algorithm NPGD and other recovery algorithms. We record the average run time to recover a single image from its compressed sensing measurements over 10 different images. All three algorithms were run on the same workstation with i7-4770K CPU, 32GB RAM and GeForce Titan X GPU. 
\begin{table}[t]
    \centering
    \begin{tabular}{||c|c|c|c||}
        \hline
         $m$ & CSGM \footnotemark   & PGD-GAN & NPGD   \\
         \hline
         200 & 5.8  & 66 & 0.09 (64x) \\
         500 & 6.6 & 60 & 0.10 (66x)  \\
         1000 & 8.0 & 63 & 0.11 (72x)  \\
         2000 & 11.2 & 61 & 0.14 (80x) \\
         \hline
    \end{tabular}
    \caption{Comparison of execution time ([sec.]) of recovery algorithms on the CelebA dataset. The relative speedup of our NPGD over the CSGM algorithm of Bora \etal is shown in parenthesis.}
    \label{tab:comp_exec_time_rec}
\end{table}
\footnotetext{Run time includes 2 initializations, as implemented by the authors, for CelebA. The same number of initializations for CelebA (and 10 for MNIST) has been used to produce results in figures \ref{fig:mnist_random}, \ref{fig:celebA_random}, \ref{fig:recovery error}, and \ref{fig:mnist_super_resolution_inpainting}. Our NPGD algorithm uses only one, deterministic initialization, $x_0 = A^Ty$.}

\subsection{Analysis: Error in Projector}
Figure \ref{fig:mnist_projector} illustrates the idempotence error of the projector for different $k$. Three different categories of images are tested, namely, MNIST training samples, MNIST test samples, and samples $G(z)$ generated using the pre-trained $G$. We use clean images from the three sources and plot the relative idempotence error $\|x-P_G(x)\|^2/\|x\|^2$. The error decreases with increasing $k$ and saturates around $k=100$. The idemopotence errors for MNIST training and test samples are very close, indicating negligible generalization error. On the other hand, samples generated by $G(z)$ give much lower errors, which indicates representation error in the GAN.  Thus we expect that a more flexible generator (deeper network) will lead to a better projector on the actual dataset and hence improve performance. 
% We computed the estimate of $\delta$-term defined in eq. \ref{eqn:delta} for our trained projector. We took samples from $G(z)$ and perturbed them with noise $\sim N(0, \sigma^2)$. 
\begin{figure}[tb]
    \centering
    \includegraphics[width=0.32\textwidth]{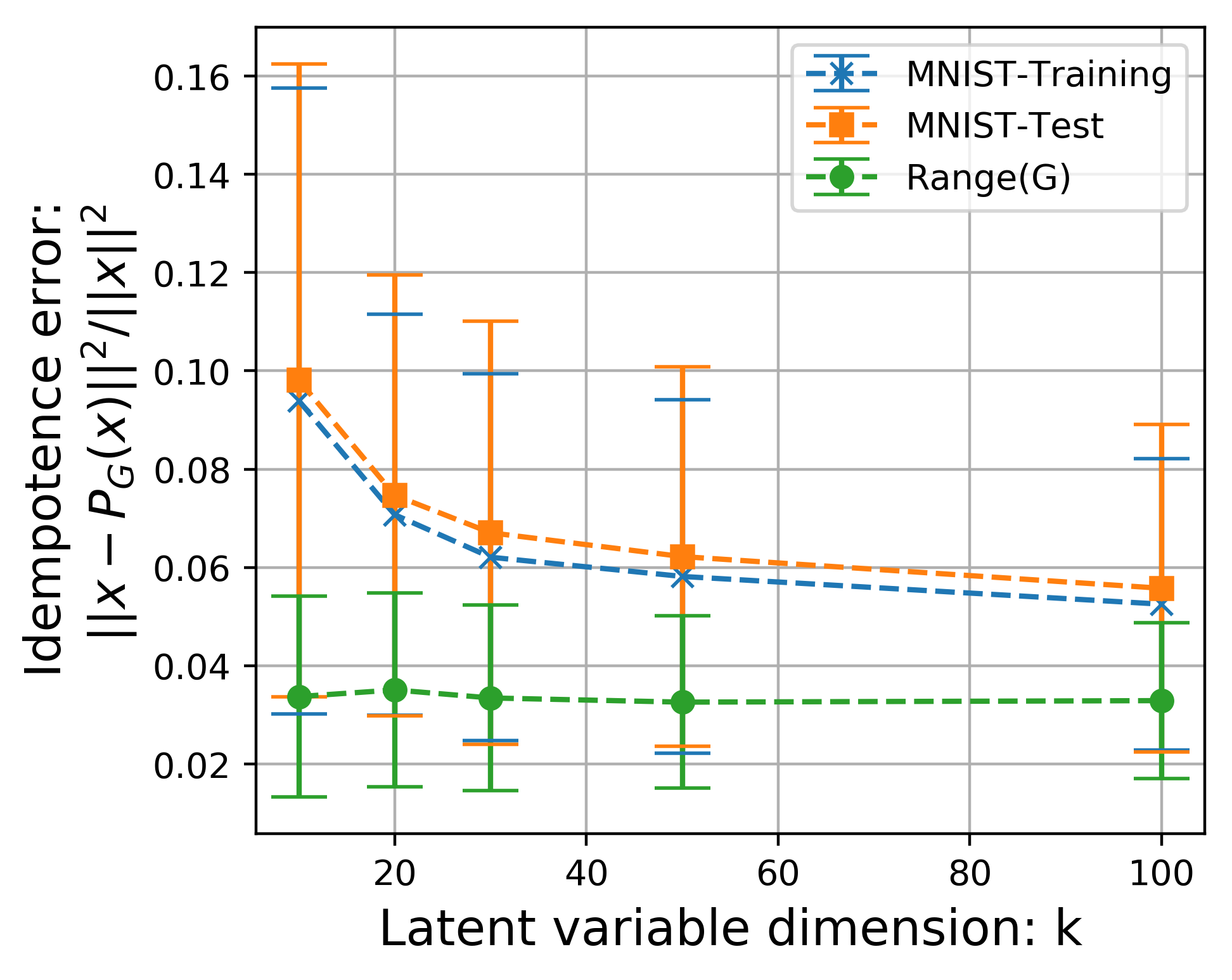}
    \caption{Idempotence Error}
    \label{fig:mnist_projector}
\end{figure}
\section{Conclusion}
In this work, we propose a GAN based projection network for faster recovery in linear inverse problems. Our method demonstrates superior performance and also provides a speed-up of $60\text{-}80\times$ over existing GAN-based methods, eliminating the expensive computation of the Jacobian matrix every iteration. We provide a theoretical bound on the reconstruction error for a moderately-conditioned measurement matrix. To help design such a matrix for compressed sensing, we propose a method which enables recovery using $5\text{-}10\times$ fewer measurements than using a random Gaussian matrix. 
Our experiments on compressed sensing, super-resolution, and inpainting demonstrate that generic linear inverse problems can be solved with the proposed method without requiring retraining. In the future, deriving a bound for the projection error $\delta$ and an associated performance guarantee is a interesting direction.   

\appendix
\setlength{\belowdisplayskip}{2pt} \setlength{\belowdisplayshortskip}{1.5pt}
\setlength{\abovedisplayskip}{2pt} \setlength{\abovedisplayshortskip}{1.5pt}
\section{Appendix: Proof of Theorem 1} 
By the assumption of $\delta$-approximate projection, 
\begin{equation}\label{eq:1.1}
    \begin{aligned}
    \|w_t-x_{t+1}\|^2 = \|w_t-G(G^\dagger(w_t))\|^2 \leq \|x^*-w_t\|^2+\delta
    \end{aligned}
\end{equation}
where from the gradient update step, we have
\begin{equation*} %\label{eq:1.2}
    w_t = x_t-\eta A^T(Ax_t-y) = x_t-\eta A^TA(x_t-x^*) \ 
\end{equation*}
Substituting $w_t$ into \eqref{eq:1.1} yields
\begin{equation*} %\label{eq:1.3}
    \begin{array} { l } { \left\| x _ { t + 1 } - x _ { t } \right\|^ { 2 } - 2 \eta \left\langle x _ { t + 1 } - x _ { t } , A ^ { T } A\left( x^* -  x _ { t } \right) \right\rangle } \\ { \leq \left\| x ^ { * } - x _ { t } \right\| ^ { 2 } - 2 \eta \|A(x^*-x_t)\|^2+\delta} \end{array}
\end{equation*}
Rearranging the terms we have
\begin{equation}\label{eq:1.4}
    \begin{aligned}
    &2 \left\langle x _ { t } - x _ { t + 1 } , A ^ { T } A\left(x ^ { * } - x _ { t } \right) \right\rangle\\
    &\leq \frac { 1 } { \eta } \left\| x ^ { * } - x _ { t } \right\| ^ { 2 } - 2 f \left( x _ { t } \right)- \frac { 1 } { \eta } \left\| x _ { t + 1 } - x _ { t } \right\| ^ { 2 } +\frac{\delta}{\eta}\\
    &\leq \Big(\frac { 1 } { \eta\alpha }-2\Big) f(x_t)- \frac { 1 } { \eta } \left\| x _ { t + 1 } - x _ { t } \right\| ^ { 2 } + \frac { \delta } { \eta }\\
    &\leq \Big( \frac { 1 } { \eta \alpha } - 2 \Big) f \left( x _ { t } \right) - \frac { 1 } { \eta\beta } \left\| Ax _ { t + 1 } - Ax _ { t } \right\| ^ { 2 } + \frac { \delta } { \eta }
    \end{aligned}
\end{equation}
where the last two inequalities follow from $REC(S,\alpha,\beta)$.
Now the LHS can be rewritten as:
%\begin{equation}\label{eq:1.5}
    \begin{align}
        &2 \left\langle x _ { t } - x _ { t + 1 } , A ^ { T } A\left( x^* - x _ { t } \right) \right\rangle \nonumber \\
        &= \|Ax ^*-Ax_{t+1}\|^2- \|Ax^*-Ax_{t}\|^2-\|Ax_{t+1}-Ax_{t}\|^2 \nonumber \\
        &=f(x_{t+1})-f(x_t) -\|Ax_{t+1}-Ax_{t}\|^2 \label{eq:1.5} 
    \end{align}
%\end{equation}
Combining \eqref{eq:1.4} and \eqref{eq:1.5}, and rearranging the terms, we have:
\begin{equation*}
\small
    %\begin{aligned}
    f(x_{t+1}) \leq \Big(\frac{1}{\eta\alpha}-1\Big) f(x_t)+\Big(1-\frac{1}{\eta\beta}\Big)\left\| A x _ { t + 1 } - A x _ { t } \right\| _ { 2 } ^ { 2 } + \frac { \delta } { \eta }
    %\end{aligned}
\end{equation*}
and since $\eta=1/\beta$,
\begin{equation*}
    f(x_{t+1})\leq \Big(\frac{\beta}{\alpha}-1\Big) f(x_t) +\beta\delta
\end{equation*}
For simplicity, we substitute $\kappa=\beta/\alpha $ in the following:
%\begin{equation}
\begin{align*}
    f(x_n)&\leq \left( \kappa - 1 \right) ^ { n } f \left( x _ { 0 } \right) +\beta\delta \sum_{k=0}^{n-1}\left( \kappa - 1 \right) ^ {k}\\
    &=\left( \kappa - 1 \right) ^ { n } f \left( x _ { 0 } \right) + \frac{\beta\left(1-(\kappa -1)^n\right)}{2-\kappa}\delta
\end{align*}
%\end{equation}
% Defining $D_n= \frac{\beta\left(1-(\kappa -1)^n\right)}{2-\kappa}$, it is easy to see that $D_n\rightarrow\frac{\beta}{2-\kappa}$. 
For convergence, we require $1 \leq \kappa= \beta/\alpha < 2$. When $n$ reaches $ \frac{1}{2-\kappa}\log\Big(\frac{f(x_0)}{C\alpha\delta}\Big)$, we have 
%\begin{equation}
    \begin{align*}
    &\|x_n-x^*\|^2\leq \frac{\|Ax_n-Ax^*\|^2}{\alpha} = \frac{f(x_n)}{\alpha}\\ 
    &\leq \left(\kappa-1\right)^n \frac{f(x_0)}{\alpha} + \frac{\beta\left(1-(\kappa -1)^n\right)}{\alpha(2-\kappa)}\delta\\
    &\leq \left(\kappa-1\right)^n \frac{f(x_0)}{\alpha} + \frac{\delta}{2/\kappa-1}\leq \Big(C+\frac{1}{2/\kappa-1}\Big)\delta
    \end{align*}
%\end{equation}
Finally, when $n\rightarrow\infty$, we have $\left(\kappa-1\right)^n \frac{f(x_0)}{\alpha}\rightarrow 0$
\begin{equation*}
    \|x^*-x_\infty\|^2 \leq \frac { \delta } { 2/\kappa - 1 } = \frac { \delta } { 2 \alpha/\beta - 1 }
\end{equation*}

\newpage
{\small
\bibliographystyle{ieee_fullname}
\bibliography{inverseGAN}
}

\end{document}